\documentclass[11pt,english]{article}
\usepackage[T1]{fontenc}
\usepackage{times}
\usepackage{fullpage}
\usepackage[title]{appendix}

\usepackage{algorithm}
\usepackage[algo2e,ruled,vlined]{algorithm2e}
\usepackage{amsmath,amssymb,amsthm,amstext,amsfonts,bbm,
    algorithmicx,
    xspace,nicefrac,
    algpseudocode
}

\usepackage{color,
enumerate,latexsym,bm,amsfonts,
  subcaption,wrapfig,verbatim,tabularx,textcomp}
\usepackage{comment} 
\usepackage{epsfig} 
\usepackage{latexsym,nicefrac,bbm}
\usepackage{xspace}
\usepackage{color,fancybox,graphicx,url}
\usepackage[shortlabels]{enumitem}
\usepackage{booktabs}
\usepackage{commath}
\usepackage{mdframed}
\usepackage{pdfsync}
\usepackage{tikz}
\usepackage{thm-restate}
\usepackage{float}

\usepackage{multicol,multirow}
\usepackage{etoolbox}
\usepackage{url}            
\usepackage{amsfonts}       
\usepackage{microtype}
\usepackage{booktabs} 
\usepackage{lipsum}
\usepackage{makecell}
\usepackage{wrapfig}
\usepackage{cellspace}    
\definecolor{darkblue}{rgb}{0,0,.75}

\usepackage{anyfontsize}

    \definecolor{DarkRed}{rgb}{0.368,0.097,0.078}
\definecolor{DarkBlue}{rgb}{0.2,0.2,0.6}

\usepackage[colorlinks,citecolor=blue,urlcolor=blue,linkcolor=blue]{hyperref}

\usetikzlibrary {positioning}
\usepackage{booktabs}

\newcommand{\conjF}{F^*}
\newcommand{\cconjF}{\overline{F}^*}

\usepackage{yfonts}
\newcommand{\gD}{\textswab{D}}

\title{Capacity-Constrained Online Learning with Delays:\\ Scheduling Frameworks and Regret Trade-offs}
\author{}
\date{}

\usepackage[
    minnames=1,maxnames=99,maxcitenames=2,
    style=alphabetic,  
    sorting=none,
    sortcites=false, 
    doi=true,
    url=true,
    uniquename=init,
    giveninits=true,
    backend=biber,
    hyperref=auto,
    natbib=true]{biblatex}

\DeclareFieldFormat{bibentrysetcount}{\mkbibparens{\mknumalph{#1}}}
\DeclareFieldFormat{labelnumberwidth}{\mkbibbrackets{#1}}
\renewbibmacro{in:}{
  \ifentrytype{article}{}{\printtext{\bibstring{in}\intitlepunct}}}

\DeclareNameAlias{sortname}{family-given}

\usepackage[hide=false,setmargin=true,marginparwidth=1.75in]{marginalia}

\usepackage[noabbrev, capitalise]{cleveref}
\crefname{cor}{Corollary}{Corallaries}
\crefname{thm}{Theorem}{Theorems}
\crefname{defn}{Definition}{Definitions}
\crefformat{equation}{Eq.~#2#1#3}
\crefrangeformat{equation}{Eqs.~#3#1#4 to~#5#2#6}

\usepackage{mathabx}

\newcommand{\rbr}[1]{\left(#1\right)}

\newcommand{\calbr}[1]{\left\{#1\right\}}

\newcommand{\trbr}[1]{(#1)}
\DeclareMathOperator*{\argmin}{argmin}

\newcommand{\pair}[1]{\langle{#1}\rangle} %for inner product

\newtheorem{theorem}{Theorem}[section]

\newtheorem{lemma}[theorem]{Lemma}

\newtheorem{definition}[theorem]{Definition}
\newtheorem{corollary}[theorem]{Corollary}

\newtheorem{fact}[theorem]{Fact}

\makeatletter
\newenvironment{algorithm-scheduler}[1][htb]{%
    \renewcommand{\ALG@name}{Scheduler}% Update algorithm name
   \begin{algorithm}[#1]%
  }{\end{algorithm}}
\makeatother

\newcommand{\E}{\mathbb{E}}
\renewcommand{\Pr}{\mathbb P}
\DeclareMathOperator*{\ind}{\mathbb I}

\newcommand{\ceil}[1]{\left\lceil\, {#1}\,\right\rceil}
\newcommand{\tfloor}[1]{\lfloor\, {#1}\,\rfloor}
\newcommand{\tceil}[1]{\lceil\, {#1}\,\rceil}

\newcommand{\tsum}{\textstyle\sum}

\newcommand\Z{\mathbb Z}
\newcommand\N{\mathbb N}
\newcommand\R{\mathbb R}
\newcommand\Rp{\mathbb R_{+}}

\newcommand{\marginlabel}[1]%
{\mbox{}\marginpar{\it{\raggedleft\hspace{0pt}#1}}}
\newlength{\pgmtab}  %  \pgmtab is the width of each tab in the
 {
	\begin{enumerate}}{\end{enumerate}}

\def\qedsketch{\ifmmode\Box\else{\unskip\nobreak\hfil
\penalty50\hskip1em\null\nobreak\hfil$\Box$
\parfillskip=0pt\finalhyphendemerits=0\endgraf}\fi}

\newlength{\tpush}
\newcommand{\handout}[5]{
   \noindent
   \begin{center}
   \framebox{ \vbox{ \hbox to \textwidth { {\bf \coursenum\ :\  \coursename} \hfill #5 }
       \vspace{3mm}
       \hbox to \textwidth { {\Large \hfill #2  \hfill} }
       \vspace{1mm}
       \hbox to \textwidth { {\it #3 \hfill #4} }
     }
   }
   \end{center}
   \vspace*{4mm}
   \newcommand{\lecturenum}{#1}
   \addcontentsline{toc}{chapter}{Lecture #1 -- #2}
}

\newcommand{\be}{\bm{e}}

\newcommand{\normop}[1]{{\left\vert\kern-0.25ex\left\vert\kern-0.25ex\left\vert #1 
		\right\vert\kern-0.25ex\right\vert\kern-0.25ex\right\vert}}
	
\newcommand{\cA}{\mathcal{A}}
\newcommand{\cB}{\mathcal{B}}

\newcommand{\cF}{\mathcal{F}}

\newcommand{\cH}{\mathcal{H}}

\newcommand{\cS}{{\mathcal{S}}}

\newcommand{\cW}{\mathcal{W}}

\newcommand{\bzero}{\bm{0}}
\newcommand{\bone}{\bm{1}}

\newcommand{\htl}{\hat{l}}
\newcommand{\wtl}{\widetilde{l}}

\newcommand{\htL}{\widehat{L}}

\newcommand{\eR}{\mathfrak{R}}

\newcommand{\wtd}{\widetilde{d}}
\newcommand{\wtL}{\widetilde L}
\newcommand{\wtO}{\widetilde{O}}

\newcommand{\htLobs}{\widehat L^{\text{obs}}}
\newcommand{\htLmiss}{\widehat L^{\text{miss}}}

\newcommand{\wtsig}{\widetilde{\sigma}}

\newcommand{\barx}{\bar x}
\newcommand{\Copt}{C_{\text{opt}}}
\newcommand{\Cutil}{C_{\text{util}}}

\newcommand{\preemptive}{preemptive\xspace}
\newcommand{\Preemptive}{Preemptive\xspace}

\newcommand{\precom}{precommitted\xspace}
\newcommand{\obsind}{observation-independent\xspace}
\newcommand{\Precom}{Precommitted\xspace}

\newcommand{\dmax}{d_{\max}}

\newcommand{\sigmax}{\sigma_{\max}}

\newcommand{\Cexp}{C_E}

\newcommand{\myeq}[1]{\stackrel{\text{(#1)}}{=}}
\newcommand{\myle}[1]{\stackrel{\text{(#1)}}{\le}}

\author{
 Alexander Ryabchenko  \thanks{
University of Toronto and Vector Institute.
}
    \and Idan Attias \thanks{Institute for Data, Econometrics, Algorithms, and Learning (IDEAL), hosted by UIC and TTIC.
    } 
    \and Daniel M. Roy \footnotemark[1]
}

\begin{document}

\maketitle

\begin{abstract}
We study online learning with oblivious losses and delays under a novel ``capacity constraint'' that limits how many past rounds can be tracked simultaneously for delayed feedback. Under ``clairvoyance'' (i.e., delay durations are revealed upfront each round) and/or ``preemptibility'' (i.e., we have ability to stop tracking previously chosen round feedback),
we establish matching upper and lower bounds (up to logarithmic terms) on achievable regret, characterizing the ``optimal capacity'' needed to match the minimax rates of classical delayed online learning, which implicitly assume unlimited capacity.  
Our algorithms achieve minimax-optimal regret across all capacity levels, with performance gracefully degrading under suboptimal capacity. For $K$ actions and total delay $D$ over $T$ rounds, under clairvoyance and assuming capacity $C = \Omega(\log(T))$, we achieve regret $\widetilde{\Theta}(\sqrt{TK + DK/C + D\log(K)})$ for bandits and $\widetilde{\Theta}(\sqrt{(D+T)\log(K)})$ for full-information feedback. When replacing clairvoyance with preemptibility, we require a known maximum delay bound $d_{\max}$, adding ${\widetilde{O}(d_{\max})}$ to the regret.  
For fixed delays $d$ (i.e., $D=Td$), the minimax regret is $\Theta(\sqrt{TK(1+d/C)+Td\log(K)})$ and the optimal capacity is $\Theta(\min\{K/\log(K),d\})$ in the bandit setting, while in the full-information feedback setting, the minimax regret is $\Theta(\sqrt{T(d+1)\log(K)})$ and the optimal capacity is $\Theta(1)$.
For round-dependent and fixed delays, our upper bounds are achieved using novel preemptive and non-preemptive scheduling policies, based on Pareto-distributed proxy delays, and batching techniques, respectively. Crucially, our work unifies delayed bandits, label-efficient learning, and online scheduling frameworks, demonstrating that robust online learning under delayed feedback is possible with surprisingly modest tracking capacity. 
\end{abstract}

\section{Introduction}

Online learning is a fundamental sequential decision-making problem in which a player repeatedly selects actions, each with some associated loss. By exploiting feedback after each action, the player aims to minimize some notion of regret, i.e.,  cumulative loss, compared to that of a class of alternative choices \citep{cesa2006prediction}. In this work, we study external regret, comparing the player's cumulative loss to that of the best action in hindsight.

The type of feedback the player receives is an important aspect of the problem. One way in which feedback can vary is by how much information is revealed about the losses. Two important types of feedback are bandit feedback \citep{lattimore2020bandit,slivkins2019introduction}, where the player learns only the loss for the action they took, and full information, where the player learns the loss for all actions. 

Another way in which feedback can vary is by when the feedback arrives. A well-studied variant considers delayed feedback, where action losses are revealed only after several rounds, forcing the player to act without immediate information about losses \citep{mesterharm2005,joulani2013,cesa16}. For example, in most recommendation systems, a platform suggests content or products to users but receives feedback (such as clicks or purchases) only after the interaction ends, requiring it to make new recommendations while relying on delayed and possibly outdated feedback.

Previous studies of online learning with delays typically assume that feedback from every round is eventually observed, even if only at the end of the game. In practice, however, resource constraints often limit how many rounds can be tracked concurrently for delayed feedback. This is particularly relevant in human-in-the-loop systems (e.g., see \cite{lykouris2024learningdefercontentmoderation}).
Consider a healthcare model where the system evaluates the long-term effects of $K$ approved drugs across a continuous stream of patients. Each treatment may result in delayed outcomes, such as recovery or side effects, that take days or weeks to manifest. Due to a limited number of highly qualified and in-demand doctors, only $C$ patients at any given time can be actively monitored via home visits, follow-ups, or remote diagnostics. As treatments progress and multiple drugs remain in use, tracking all $T$ patients until their outcomes becomes infeasible, requiring up to $\Omega(T)$ tracking slots in the worst case. With only $C$ slots, the system must reallocate attention dynamically, releasing patients with slow outcomes to accommodate new ones, in order to ensure a steady flow of informative feedback. This underscores the need for capacity-aware monitoring policies that manage exploration under resource constraints.

As another example, consider a recommendation system operating on a massive user stream, where delays arise from concurrently handling many users. Observing user behavior often requires maintaining an open session for each user until a conversion event (e.g., a purchase) occurs. However, since only a limited number of sessions can be maintained at once, it becomes impossible to track every user. Thus, resource constraints naturally arise in this scenario as well.

Motivated by these examples, we propose a resource-efficient version of online learning with delays, where in order to observe feedback from a particular round, that round must be continuously tracked until its delay period ends, and the number of rounds tracked simultaneously for delayed feedback is capped by a specified limit $C$, which we term \textit{capacity}. We refer to this broader framework as \textit{Delay Scheduling}, drawing an analogy to Online Job Scheduling (e.g., see \textcite{borodin1998}), a problem that involves assigning sequentially arriving jobs between multiple resources with the goal of optimizing specific objective, such as maximizing the number of completed jobs.
In analogy to Online Job Scheduling, several variations arise naturally in our Delay Scheduling framework:
\begin{itemize}[leftmargin=0.5cm, itemsep=-0.15cm,before=\vspace{-0.2cm}, after=\vspace{-0.2cm}]
    \item \textbf{Clairvoyant VS Non-Clairvoyant:} If at the start of each round, the player observes this round's delay ($d_t$), we refer to this as the \textit{clairvoyant} framework. In contrast, if the player discovers this delay only when the feedback actually arrives (at the end of round $t+d_t$), we call it \textit{non-clairvoyant}.
    Delayed bandits under clairvoyance were previously studied by \cite{thune}, who leveraged this upfront information about delays to eliminate the need for prior knowledge of both the time horizon $T$ and total delay $D$, while also removing the assumption of bounded delays.

    \item \textbf{\Preemptive VS Non-\Preemptive:} If the framework allows the player to stop tracking rounds before receiving their feedback, we define it as \textit{\preemptive}; otherwise, it is \textit{non-\preemptive}. Importantly, once we stop tracking a round (i.e., preempt it), we cannot resume tracking it later,\footnote{
    In this paper, ``preemption'' refers to permanently stopping tracking a round before feedback arrives (Cf.\ ``revoking'' \cite{borodin2023}). In Online Scheduling, preemption may refer to pausing a task with the possibility of resumption or restarting.} aligning our framework more closely with Online Interval Scheduling \cite{lipton1994, woeginger1994}.
\end{itemize}

While job scheduling typically optimizes metrics such as throughput, makespan, or latency, Delay Scheduling is an online learning problem where the goal is to minimize regret by strategically allocating resources in order to observe representative feedback. This introduces a novel challenge: balancing exploration and timely feedback collection while deciding which rounds to track under limited capacity. Consequently, standard scheduling techniques do not directly apply to our setting.

Prior work on Delayed Online Learning implicitly leveraged scheduling concepts to enhance algorithm performance, with \cite{thune,seldin20} proposing ``skipping schemes'' (similar to preemption) to exclude rounds with excessive delays to improve regret bounds. However, the impact of limited capacity remained unexplored.
For additional related work, see Appendix~\ref{app:additional-related-work}.

\subsection{Problem Setting}

In the Delay Scheduling game (\Cref{fig:game-desc}), the player interacts with an environment determined by an oblivious adversary over $T$ rounds. Before the game begins, the adversary sets delays and losses for each round. Without loss of generality, we assume that $\dmax \le T$ because the feedback of any round $t$ with $t+d_t > T$ will not be observed. The player repeatedly selects actions from a fixed set while maintaining a \textit{tracking set} of at most $C$ round indices (the ``capacity constraint''). The player receives feedback only for rounds that are currently in the tracking set and may modify this set according to the rules of the specific game variation being played.

\begin{figure}[H]
\centering
\fbox{\parbox{0.9\textwidth}{
    \textbf{Delay Scheduling Game}\vspace{5pt}\\
    $\bullet$ \textit{Visible Parameters:} number of actions $K$ and capacity $C$.\\
    $\bullet$ \textit{Latent Parameters:} number of rounds $T$.\\
    $\bullet$ \textit{Pre-game:} adversary selects losses $l_{t} \in [0, 1]^K$ and delays $d_t \in \Z_{\ge 0}$ for all $t \in [T]$.
    
    \vspace{5pt}
    \noindent Player initializes empty tracking set $S$ of maximum size $C$.\\
    \noindent For each round $t=1, 2, \ldots, T$:
    \begin{enumerate}[leftmargin=1cm,noitemsep,before=\vspace{-0.3cm},after=\vspace{-0.3cm}]
        \item[0.] If the framework is \textbf{clairvoyant}, then the environment reveals delay $d_t$.  
        \item The player selects action $A_t\in \{1, ..., K\}$, plays it, and incurs corresponding loss $l_{t,A_t}$.
        \item The player may add round index $t$ to the tracking set $S$, provided $|S| < C$.  
        \item For all $s \le t$ such that $s+d_s = t$ and $s \in S$, the environment reveals 
        \begin{itemize}[noitemsep,before=\vspace{-0.1cm},after=\vspace{-0.1cm}]
            \item round-value pair $(s, l_{s,A_s})$ in the \textbf{multi-armed bandit} game,
            \item round-vector pair $(s, l_s)$ in the \textbf{full-information} game,
        \end{itemize}
        and $s$ is automatically
        removed from the tracking set $S$.
        \item If the framework is \textbf{\preemptive}, the player may remove elements (possibly none) from $S$.
    \end{enumerate}
}}
\caption{The Delay Scheduling Game in all 
variations: clairvoyant vs.\ non-clairvoyant, \preemptive vs.\ non-\preemptive, and full-information vs.\ bandit feedback.}
\label{fig:game-desc}

\end{figure}

Since delays are assigned to rounds rather than round-action pairs, we track rounds using $C$ units of \textit{round-based} capacity. In the full-information regime, each unit of capacity tracks all $K$ losses from the corresponding round. In the bandit regime, it tracks only the loss of the selected action.

The player's objective is to minimize \textit{expected regret}, 
$\eR_T = \E[\tsum_{t=1}^T l_{t,A_t}] - \min_{i\in[K]} \tsum_{t=1}^T l_{t,i}$,
i.e., the player's expected cumulative loss in excess of that of the best single action in hindsight, where the expectation is taken over the player's actions.

\paragraph{Notation.} For all $n \in \N$, define $[n] = \{1, \ldots, n\}$. For all $z\in\Z$, let $\Z_{\ge z} = \Z \cap [z, \infty)$. For each $i \in [K]$, let $\be_i \in \R^K$ denote the standard basis vector, where $(\be_i)_j = \ind(i = j)$ for all $j \in [K]$. Let $\bzero_K, \bone_K \in \R^K$ be the zero and one vectors, i.e., $(\bzero_K)_j = 0$ and $(\bone_K)_j = 1$ for all $j \in [K]$. Define the probability simplex over $[K]$ as $\Delta([K]) = \{x \in \Rp^K : \norm{x}_1 = 1\}$.

\subsection{Our Contributions}

Our key innovation is a capacity-efficient approach to tracking delayed rounds, without compromising regret performance. While prior work in Delayed Online Learning requires tracking all delayed rounds (i.e., the utilized capacity, $\Cutil$, adjusts to demand and can be arbitrarily large), we introduce selective sampling policies that maintain at most a constant-sized (in $T$) set of rounds to track when delays are fixed (i.e., all $d_t = d$), and at most a logarithmic-sized  one, when delays are round-dependent (i.e., $d_t$ can be arbitrary), achieving a significant reduction in tracking-set size, without degrading regret guarantees. This addresses the fundamental question of the optimal capacity, $\Copt$, sufficient to match the asymptotic regret of Delayed Online Learning.
We analyze Delay Scheduling across various settings characterized by three dimensions: delay structure (fixed or round-dependent), delay knowledge at action time (clairvoyant or non-clairvoyant), and scheduling flexibility (\preemptive or non-\preemptive), for both bandit and full-information regimes.

As is standard in Delayed Online Learning, our regret bounds depend on the number of actions $K$, the time horizon $T$, and the total delay $D = \sum_{t=1}^T d_t$. In this work, we additionally study the dependence of regret on the capacity $C$. Another important quantity we consider is the number of outstanding delays, $\sigma_t = \sum_{s=1}^{t-1} \ind(s + d_s \ge t)$ for each round $t$. Letting $\sigmax = \max_{t}\sigma_t$ and $\dmax = \max_{t} d_t$, we note that a capacity of order $\Omega(\sigmax)$ is sufficient to observe feedback from every round. While $\sigmax$ can be as large as $\Omega(\min\{\sqrt{D},\dmax\})$, we show that, in most cases, capacity of this order is unnecessary.

\paragraph{Delay Scheduling with Fixed Delays.}
We first consider fixed delays, introduced in the bandit setting by \textcite{cesa16} and in the full-information setting by \textcite{weinberger2002}. Here, all delays are equal, i.e., $d_t = d$, and known in advance, naturally corresponding to the clairvoyant framework. We study both \preemptive and non-\preemptive versions of this setting and show that there is no benefit in allowing preemption: our lower bound holds for the \preemptive case, and our upper bound algorithm applies to both frameworks.
We determine the minimax regret in bandit and full-information regimes (\Cref{tab:results-fixed-delays}).

\begin{table}[h] \centering 
\resizebox{0.9\columnwidth}{!}{
\begin{tabular}{c|c|c|c} 
\hline
\multicolumn{4}{c}{\textbf{Delayed Online Learning with fixed delays}}\\
\hline
Regime & Regret Bounds & Utilized capacities & Reference\\ 
\hline 
Bandit & $\Theta\rbr{\sqrt{TK} + \sqrt{Td\log(K)}}$ & $\Cutil = \Theta(d)$ & \cite{cesa16, seldin20}\\
Full-info & $\Theta\rbr{\sqrt{T(d+1)\log(K)}}$ & $\Cutil=\Theta(d)$ & \cite{weinberger2002}\\
\hline
\hline
\multicolumn{4}{c}{\textbf{Delay Scheduling with fixed delays}}\\
\hline
Regime & Regret Bounds & Optimal capacities & Reference\\ 
\hline
Bandit & $\Theta\rbr{\sqrt{TK (1 + d/C)} + \sqrt{Td\log(K)}}$ & $\Copt=\Theta\left(\min\{\tfrac{K}{\log(K)},d\}\right)$ & \multirow{2}{*}{Theorem~\ref{thm:fixed-delays}}\\
Full-info & $\Theta\rbr{\sqrt{T(d+1)\log(K)}}$ & $\Copt=\Theta(1)$ &\\
\hline
\end{tabular} 
}
\caption{Minimax regret bounds for Delay Scheduling compared to Delayed Online Learning under the assumption of fixed delays.
}   
\label{tab:results-fixed-delays}
\end{table}
While $C = d+1$ is the exact capacity required to observe feedback from every round under fixed delays, we establish that capacities of $C = \Omega(\min\{K / \log(K), d\})$ and $C = \Omega(1)$ are both sufficient and necessary to eliminate the impact of the capacity constraint in the bandit and full-information settings, respectively. This stands in contrast to previously studied delayed algorithms, which implicitly required a capacity of $\Omega(d)$.

\paragraph{Delay Scheduling: Clairvoyant and Non-\Preemptive.} 
In this setting, the player observes the delay $d_t$ at the start of each round $t$ and must decide whether to track it, with no option to preempt once committed. When $C = \Omega(\log T)$, we establish minimax-optimal upper bounds for this setting (up to logarithmic factors), matching the fixed-delay case with $D = Td$ (\Cref{tab:results-CNP-general}).
\begin{table}[H] \centering 
\resizebox{1\columnwidth}{!}{
\begin{tabular}{c|c|c|c} 
\hline
\multicolumn{4}{c}{\textbf{Delayed Online Learning with round-dependent delays}}\\
\hline
Regime & Regret Bounds & Utilized capacities & Reference\\ 
\hline 
Bandit & $\Theta\rbr{\sqrt{TK} + \sqrt{D\log(K)}}$ & $\Cutil = \Theta(\sigmax)$ & \cite{cesa16, seldin20}\\
Full-info & $\Theta\rbr{\sqrt{(D + T)\log(K)}}$ & $\Cutil=\Theta(\sigmax)$ & \cite{weinberger2002,joulani2016}\\
\hline
\hline
\multicolumn{4}{c}{\textbf{Clairvoyant Non-\Preemptive Delay Scheduling with round-dependent delays}}\\
\hline
Regime & Regret Bounds & Optimal capacities & Reference\\ 
\hline
\multirow{2}{*}{Bandit} & $O\rbr{\sqrt{TK + \tfrac{\log(T)}{C}(D+T)K} + \sqrt{D\log(K)}}$ & $\Copt=O\left(\log(T)\cdot \min\{\frac{K}{\log(K)}, \frac{D}{T} + 1\}\right)$ & Theorem~\ref{thm:CNP-results}\\
& $\Omega\rbr{\sqrt{TK + DK/C} + \sqrt{D\log(K)}}$ & $\Copt=\Omega\left(\min\{\frac{K}{\log(K)}, \frac{D}{T}\}\right)$ & Theorem~\ref{thm:fixed-delays}\\
\hline
\multirow{2}{*}{Full-info} & $O\rbr{\sqrt{(1+\tfrac{\log(T)}{C})(D + T)\log(K)}}$  & $\Copt=O(\log(T))$ & Theorem~\ref{thm:CNP-results}\\
& $\Omega\rbr{\sqrt{(D+T)\log(K)}}$ & $\Copt=\Omega(1)$ & Theorem~\ref{thm:fixed-delays}\\
\hline
\end{tabular} 
}
\caption{Minimax regret bounds for Delay Scheduling, assuming $C = \Omega(\log(T))$ for upper bounds, compared to Delayed Online Learning.
}   
\label{tab:results-CNP-general}
\end{table}
While $C = \sigmax+1$ is the exact capacity required to observe feedback from every round under round-dependent delays, we establish that capacities $C = \Omega(K\log(T) / \log(K))$, and $C = \Omega(\log(T))$ are sufficient to avoid the impact of the capacity constraint in the bandit and full-information settings, respectively. In general, these capacity requirements are significantly smaller than $\Theta(\sigmax)$, which can range from $\Omega(D/T)$ to $O(\sqrt{D})$ for round-dependent delays. 

The results in Table~\ref{tab:results-CNP-general} are derived from a more general bound that holds for all $C \ge 1$; however, to achieve this bound our algorithm requires prior knowledge of the magnitude of $T^{1/C}$ in order to set its parameters (see Table~\ref{tab:results-general}). In particular, when $C = \Omega(\log(T))$, we have $T^{1/C} = O(1)$.

\paragraph{Delay Scheduling: Non-Clairvoyant and \Preemptive.}
In this setting, the player can preempt rounds, but delays remain hidden at action times. The player observes each delay only for as long as it stays in the tracking set up to the current time. This is more restrictive than standard Delayed Online Learning (without clairvoyance), where all delays are continuously observed up to the current time.
Without prior knowledge of $T$ and $D$, but assuming that an upper bound on the maximum delay, $\dmax$, is known, we establish bounds identical to those in the Clairvoyant Non-\Preemptive setting, up to an additional $\widetilde O(\dmax)$ term. Specifically, when $C = \Omega(\log(T))$, we establish the regret bounds in \Cref{tab:results-NCP-general}. As in the Clairvoyant Non-Preemptive framework, a more general bound exists that requires prior knowledge of the magnitude of $T^{1/C}$, up to an additional $\widetilde O(\dmax)$ term. 

\begin{table}[ht] \centering
\resizebox{0.95\columnwidth}{!}{
\begin{tabular}{c|c|c} 
\hline
\multicolumn{3}{c}{\textbf{Non-Clairvoyant \Preemptive Delay Scheduling for $C = \Omega(\log(T))$ with known $\dmax$}} \\ \hline 
Regime & Regret Bounds & Reference\\ \hline
Bandit & $O\rbr{\sqrt{TK + \tfrac{\log(T)}{C}(D+T)K} + \sqrt{D\log(K)}} + \widetilde O\rbr{\dmax\sqrt{1 + \tfrac{K}{C}}}$ &  \multirow{2}{*}{Theorem~\ref{thm:NCP-results}}\\ 
Full-info & $O\rbr{\sqrt{(1+\tfrac{\log(T)}{C})(D + T)\log(K)}} + \widetilde O(\dmax)$ & \\ \hline 
\end{tabular}
}
\caption{Regret upper bounds for Non-Clairvoyant \Preemptive Delay Scheduling with round-dependent delays when $C = \Omega(\log(T))$, assuming prior knowledge of $\dmax$.} 
\label{tab:results-NCP-general}
\end{table}

As in the Clairvoyant Non-Preemptive framework, a more general bound exists that requires prior knowledge of the magnitude of $T^{1/C}$. In this Non-Clairvoyant Preemptive framework, the regret bound additionally includes a $\widetilde{O}(\dmax)$ term (see Table~\ref{tab:results-general}).

Assuming prior knowledge of $D$, preempting rounds with delays exceeding $\sqrt{D}$ removes dependence on $\dmax$, matching the Clairvoyant Non-\Preemptive regret bound in Table~\ref{tab:results-CNP-general}. Prior work has explored various adaptive ``skipping schemes'' to mitigate the impact of highly unbalanced delays, treating skipped rounds as contributing at most 1 to regret while ignoring their delays. For example, such adaptive techniques may optimize regret by selecting the optimal skipping threshold (e.g., \cite{thune} under clairvoyance) or by choosing the best subset of rounds to skip (e.g., \cite{seldin20}). However, Non-Clairvoyant \Preemptive Delay Scheduling imposes strict constraints on observing information about delays during runtime, preventing direct application of these techniques.

\paragraph{Delay Scheduling: Non-Clairvoyant and Non-\Preemptive.} For completeness, we also consider the most restrictive setting, where the player has to commit to tracking rounds without the ability to preempt and without any clairvoyant knowledge of delays. Assuming prior knowledge of both $T$ and $D$, we are still able to achieve sublinear regret in both bandit and full-information regimes (\Cref{tab:results-NCNP-known-D}) for all $C\ge 1$.
\begin{table}[H] \centering 
\resizebox{0.95\columnwidth}{!}{
\begin{tabular}{c|c|c} 
\hline
\multicolumn{3}{c}{\textbf{Non-Clairvoyant Non-\Preemptive Delay Scheduling with known $T, D$}} \\ \hline 
Regime & Regret Bounds & Optimal capacities\\ 
\hline
Bandit & $O\rbr{\sqrt[3]{\tfrac{T(D+T)K}{C}} + \sqrt{TK + D\log(K)}}$ & $\Copt = O\rbr{\tfrac{K}{\log(K)} \cdot \tfrac{T}{\sqrt{(D+T)\log(K)}}}$ \\ \hline 
Full-info & $O\rbr{\sqrt[3]{\tfrac{T (D+T) \log(K)}{C}} + \sqrt{(D+T)\log(K)}}$ & $\Copt = O\rbr{\tfrac{T}{\sqrt{(D+T)\log(K)}}}$ \\ \hline 
\end{tabular}}
\caption{Regret upper bounds for Non-Clairvoyant Non-\Preemptive Delay Scheduling when capacity $C \ge 1$, assuming prior knowledge of $T$ and $D$. Derived from Corollary~\ref{cor:NCNP-known-TD}.} 
\label{tab:results-NCNP-known-D}
\end{table}
Thus, given prior knowledge of $T$ and $D$, a capacity of order $\Omega(\sqrt{T})$ with respect to $T$ is always sufficient to avoid the effects of limited resources. Furthermore, as the total delay $D$ increases, our upper bound on the optimal capacity decreases, ultimately reaching $O(1)$ for $D = \Omega(T^2)$\footnote{This may seem counterintuitive; however, when $D$ is of the order $T^2$, linear regret becomes unavoidable in Delayed Online Learning, indicating that the optimal capacity should be of the smallest order $O(1)$.}.

Alternatively, when only an upper bound on the maximum delay, $\dmax$, is available, we establish different regret bounds (\Cref{tab:results-NCNP-known-dmax}). In this case, ensuring sublinear regret may require $C$ to grow polynomially with $T$ when $\dmax = \Omega(T)$.

\begin{table}[H] \centering 
\resizebox{0.95\columnwidth}{!}{
\begin{tabular}{c|c|c} 
\hline
\multicolumn{3}{c}{\textbf{Non-Clairvoyant Non-\Preemptive Delay Scheduling with known $\dmax$}} \\ \hline 
Regime & Regret Bounds & Optimal capacities\\ 
\hline
Bandit & $O\rbr{\sqrt{\tfrac{T \dmax K}{C}} + \sqrt{TK + D\log(K)}}$ & $\Copt = O\rbr{\min\calbr{\tfrac{K}{\log(K)}\cdot \tfrac{T\dmax}{D+T}, \dmax}}$\\ \hline
Full-info & $O\rbr{\sqrt{\frac{T \dmax\log(K)}{C}} + \sqrt{(D+T)\log(K)}}$ & $\Copt = O\rbr{\tfrac{T\dmax}{D+T}}$\\ \hline 
\end{tabular}}
\caption{Regret upper bounds for Non-Clairvoyant Non-\Preemptive Delay Scheduling when capacity $C \ge 1$, assuming prior knowledge of $\dmax$. Derived from Corollary~\ref{cor:NCNP-known-dmax}.}  
\label{tab:results-NCNP-known-dmax}
\end{table}
For proofs and additional details about the setting, see Appendix~\ref{app:NCNP-scheduling}.

\paragraph{Delay Scheduling under the expectation-capacity constraint.} As an alternative approach to Delay Scheduling, we consider a setting where only the \emph{expected size} of the tracking set is required to remain bounded by the expectation-capacity $\Cexp \in (0, \infty)$ at each round. We refer to this constraint as the ``expectation-capacity constraint'' and explore it further in Appendix~\ref{app:results-expected-capacity}. Notably, we establish minimax bounds on achievable regret for all values of $\Cexp$ (up to logarithmic factors). With prior knowledge of $\log(T)$ up to constant multiplicative factors, our algorithms for the standard capacity constraint can be adapted to the expectation-capacity setting, achieving similar regret bounds formula-wise as in Tables~\ref{tab:results-CNP-general} and \ref{tab:results-NCP-general}, but with $\Cexp$ replacing $C$ in the bounds and without assuming $\Cexp = \Omega(\log(T))$ (see Table~\ref{tab:results-expected-capacity}). We prove matching lower bounds in Theorem~\ref{thm:expected-capacity-lb} of Appendix~\ref{app:results-expected-capacity}, completing the theoretical characterization of the problem.

\subsection{Technique Highlights}

Our paper introduces several key technical advances, listed here in the order they appear in the text. 

The core learning component in our algorithms is an FTRL-based framework for delayed online learning. We extend the Delayed FTRL algorithm of \citet{seldin20} to accommodate loss scales that vary between rounds (\Cref{subsec:FTRL-no-scale}). When applied to Delay Scheduling, these scales reflect the weighting of losses with respect to probabilities of observing them.

We then introduce several scheduling techniques integrated with the learning algorithm. In Section~\ref{sec:batching}, we present a natural Batch Partitioning method (Algorithm~\ref{alg:batch-partition}), which achieves minimax regret in Delay Scheduling under fixed delays, with matching lower bounds established via a reduction from Label-Efficient and Delayed Online Learning. This lower bound extends to round-dependent delays, which we match up to logarithmic factors in the following sections. 

Section~\ref{sec:scheduling-and-learning} introduces schedulers as autonomous subroutines, thereby externalizing scheduling from learning. 
Within this framework, we establish general regret bounds in Theorem~\ref{thm:ftrl-scheduler} for a specific class of \textit{\precom} schedulers paired with Delayed FTRL (\Cref{alg:FTRL-with-scheduler}). For this class of \precom schedulers, we introduce \preemptive and non-\preemptive variants (Schedulers \ref{alg:scheduler-bernoulli} and \ref{alg:scheduler-proxy-delays}), with corresponding regret bounds established in \Cref{thm:NCP-results,thm:CNP-results}. Notably, the \preemptive Scheduler~\ref{alg:scheduler-proxy-delays} introduces a novel technique of sampling \emph{proxy delays} to balance the trade-off between observing long-delay feedback and controlling the size of the tracking set.

\section{Delayed Follow the Regularized Leader with Time-Varying Loss Scales}\label{subsec:FTRL-no-scale}

As the first algorithm achieving minimax regret for oblivious bandits with round-dependent delays, the Delayed FTRL algorithm by \citet{seldin20} uses a hybrid, time-varying regularizer $F_t$, where each action $A_t$ is sampled according to 
$\textstyle x_t = \argmin_{x \in \Delta([K])} \pair{x, \htLobs_t} + F_t(x)$.
Here, $\htLobs_t$ denotes the cumulative estimator of all previously observed loss vectors up to the start of round $t$, and regularizer $F_t(x) = \alpha_t^{-1} F_{\text{Ts}}(x) + \beta_t^{-1} F_{\text{NE}}(x)$ is a weighted sum of $\frac{1}{2}$-Tsallis entropy $F_{\text{Ts}}(x) = -\sum_{i=1}^K 2x_i^{1/2}$ and negative entropy $F_{\text{NE}}(x) = \sum_{i=1}^K x_i \log(x_i)$, with separate learning rates $\alpha_t$ and $\beta_t$ for each part. For the full-information regime, we disable the Tsallis component by setting $\alpha_t = \infty$.

We extend this Delayed FTRL to a setting where loss scales $(B_t)_{t=1}^T \in [0, \infty)$ vary across rounds, with an oblivious adversary selecting losses $l_t \in [0, B_t]^K$ and delays $d_t \in \Z_{\ge 0}$ before the game begins.\footnote{Unlike in the scale-free online learning literature (e.g., see \cite{orabona2016scalefreeonlinelearning, putta2021scalefreeadversarialmulti}), our analysis is independent of how the player observes $B_t$, as we study this Delayed FTRL only with fixed, pre-determined sequences of learning rates.} 
This framework serves as the foundation for reductions from other algorithms in the following sections, providing a unified approach to bounding regret across various settings.
For each round $t$, let $\cW_t = \{s \in [t-1] : s + d_s \ge t\}$ denote the \emph{working set} of rounds with pending feedback.

\begin{theorem}\label{thm:FTRL-no-scale}
    Consider Delayed FTRL with time-varying loss scales $l_t \in [0, B_t]^K$ (formally, Algorithm~\ref{alg:FTRL-no-scale}) running with arbitrary non-increasing sequences of learning rates $(\alpha_t)_{t=1}^T$ and $(\beta_t)_{t=1}^T$. Then, the regret in the bandit regime satisfies: 
    \begin{equation*}
        \eR_{T} \le \tsum_{t=1}^T \rbr{\sqrt{K}\alpha_t B_t^2 + \beta_t B_t \tsum_{s\in\cW_t}B_s} + 2\sqrt{K}\alpha_T^{-1} + \log(K)\beta_T^{-1}. 
    \end{equation*}
    And in the full-information regime ($\alpha_t = \infty$): 
    \begin{equation*}
        \eR_{T} \le \tsum_{t=1}^T  \rbr{\beta_t B_t^2 + \beta_t B_t \tsum_{s\in\cW_t}B_s} + \log(K)\beta_T^{-1}. 
    \end{equation*}
\end{theorem}
We prove Theorem~\ref{thm:FTRL-no-scale} in Appendix~\ref{app:delayed-ftrl-proof} by expanding the proof of Theorem 3 in \cite{seldin20} to handle time-varying loss scales in both bandit and full-information regimes.

\section{Batch Partitioning Algorithm with a Notable Application for Fixed Delays}\label{sec:batching}

A natural approach to managing the tracking set under the capacity constraint is to partition rounds into contiguous batches of equal size and track a single uniformly selected \textit{representative} round per batch. If batch size $b$ is sufficiently large (e.g., $b\ge \frac{\dmax}{C-1}$), then capacity $C$ is never exceeded. 

The batching technique has been explored in various online learning settings (e.g., \cite{dekel12, arora12}). In the presence of delayed feedback, batching is particularly effective, as it decouples the impact of one-per-batch feedback observation from the effect of delays on the regret. A similar decoupling phenomenon was independently leveraged in recent work by \cite{wan2024improvedregretbanditconvex} to obtain improved regret bounds for bandit convex optimization. We introduce a novel use of batching as a delay scheduling mechanism to satisfy the capacity constraint throughout the game. 

In Algorithm~\ref{alg:batch-partition}, we run Delayed FTRL at the batch level: selecting one action per batch to be used in all its rounds and updating the player’s decision rule using aggregated loss estimates from the observed representative rounds. Notably, this algorithm can be run in the most restrictive non-clairvoyant and non-\preemptive framework.

\begin{algorithm}[ht]
\caption{Delay FTRL with Batch Partitioning}\label{alg:batch-partition}
\SetKwInOut{Input}{Input}
\Input{Number of actions $K$ and capacity $C$.}
\SetKwInOut{Parameters}{Parameters}
\Parameters{Batch size $b$.}
\begin{enumerate}[leftmargin=0.45cm,noitemsep,before=\vspace{-0.4cm},after=\vspace{-0.5cm}]
    \item Initialize empty tracking set $S$ of maximum size $C$. Initialize $\htLobs_{1} = \bzero_K$.
    \item \textbf{For} batch $\tau = 1, 2, ..., \tceil{T/b}$:
    \begin{enumerate}[before=\vspace{-0.1cm},noitemsep,after=\vspace{-0.1cm}]
        \item Sample batch-representative $u_{\tau} \sim \text{Unif}\{(\tau-1) b + 1,  ..., \tau b\}$. 
        \item Calculate learning rates $\alpha_{\tau}$, $\beta_{\tau}$ using available information and sample action $A_{\tau}^b \in [K]$ according to $x_{\tau}^b = \argmin_{x\in\Delta([K])}\pair{x, \htLobs_{\tau}} + \alpha_{\tau}^{-1}F_{\text{Ts}}(x) + \beta_{\tau}^{-1}F_{\text{NE}}(x)$.
        \item \textbf{For} round $t = (\tau-1)b + 1, ..., \min\{\tau b, T\}$:
        \begin{itemize}
            \item Play action $A_t = A_{\tau}^b$. If round $t$ is a representative $u_{\tau}$, then add $t$ to $S$. 
            \item For each expired $u_s \in S$ (i.e., $u_s + d_{u_s} = t$), observe feedback, and set estimator $\htl_{u_s}$:
            \begin{itemize}[noitemsep, after=\vspace{0.1cm}]
                \item In the bandit regime, observe $(u_s, l_{u_s,A_{u_s}})$ and set $\htl_{u_s} = l_{u_s,A_{u_s}} x_{u_s,A_{u_s}}^{-1} \be_{A_{u_s}}$.
                \item In the full-information regime, observe $(u_s, l_{u_s})$ and set $\htl_{u_s} = l_{u_s}$. 
            \end{itemize}
        \end{itemize}
        \item Update $\htLobs_{\tau+1} = \htLobs_{\tau} + \sum_{s: (\tau-1)b < u_s + d_{u_s} \le \tau b} \htl_{u_s}$.
    \end{enumerate}
\end{enumerate}
\end{algorithm}

The following batch-level notation arises naturally. The number of batches is $T' = \tceil{\frac{T}{b}}$. Each batch\footnote{The final batch is extended to $b$ elements by padding with rounds of zero loss and delay.} $\tau \in [T']$ has a representative $u_{\tau}$, batch loss $l^b_{\tau} = l_{u_{\tau}}$, batch delay $d^b_{\tau} = \tceil{\frac{u_{\tau} + d_{u_{\tau}}}{b}} - \tceil{\frac{u_{\tau}}{b}}$, number of outstanding batch delays  $\sigma^b_{\tau} = |\{s < \tau : s + d^b_s \ge \tau\}|$, and total delay $D_{\tau}^b = \sum_{s\in [\tau]} \sigma^b_s$.

\begin{theorem}\label{thm:batch-bandit-fullinfo}
    Let Algorithm~\ref{alg:batch-partition} be run with batch size $b \ge \frac{\dmax}{C-1}$. Then, with learning rates $\alpha_{\tau} = \sqrt{1/\tau}, \beta_{\tau} = \sqrt{\frac{\log(K)}{D_{\tau}^b}}$, the regret in the bandit regime satisfies: 
    \begin{equation*}
        \eR_T \le 14\sqrt{TbK} + 3\sqrt{D\log(K)}.
    \end{equation*}
    In the full-information regime, with learning rate $\beta_{\tau} = \sqrt{\frac{\log(K)}{\tau + D_{\tau}^b}}$: 
    \begin{align*}
        \eR_T \le 12\sqrt{Tb\log(K)} + 3\sqrt{D\log(K)}.
    \end{align*}
\end{theorem}

\begin{theorem}\label{thm:fixed-delays}
    Consider Delay Scheduling with fixed delays $d_t = d$. Across all scheduling frameworks regardless of preemptibility and clairvoyance, the minimax regret is $\Theta(\sqrt{TK(1+d/C)} + \sqrt{Td\log(K)})$ for the bandit regime and $\Theta(\sqrt{T(d+1)\log(K)})$ for the full-information regime. 
\end{theorem}
The proofs of \Cref{thm:batch-bandit-fullinfo,thm:fixed-delays} are in Appendix~\ref{app:batching-proof}. For Theorem~\ref{thm:batch-bandit-fullinfo}, the proof reduces the original problem to a batch-level game by conditioning on representative round selection. Upon conditioning, we obtain a $T'$-round game with losses $l^b_{\tau}$ and delays $d^b_{\tau}$ on which Algorithm~\ref{alg:batch-partition} effectively runs the Delayed FTRL from Section~\ref{subsec:FTRL-no-scale}, for which we apply Theorem~\ref{thm:FTRL-no-scale} to obtain the stated bounds. The upper bound for Theorem~\ref{thm:fixed-delays}  follows from Theorem~\ref{thm:batch-bandit-fullinfo} by setting $b = \max\{1, \tceil{\frac{d}{C-1}}\}$, while the lower bound uses reductions from classical label efficient and delayed online learning settings.

\section{General Paradigm for Scheduling and Learning}\label{sec:scheduling-and-learning}
We propose a general paradigm that separates the \textit{scheduling policy} (scheduler), which manages how rounds are tracked for delayed feedback, from the \textit{learning algorithm}, which is responsible for action selection. This approach is inspired by the principle of separating high-level policies from low-level mechanisms \cite{levin1975}, a widely used concept in operating systems design. We assume that the scheduler operates autonomously, managing the tracking set and delivering observations to the learning algorithm, as specified by the following abstract interface:
\begin{center}
\fbox{\parbox{0.95\textwidth}{
    \textbf{Abstract Scheduler Interface}\\
    \noindent \textbf{Input:} Empty \textit{tracking set} $S$ of maximum size $C$.\\
    \noindent For each round $t=1, 2, \ldots, T$:
    \begin{enumerate}[leftmargin=1cm,noitemsep,before=\vspace{-0.3cm},after=\vspace{-0.3cm}]
        \item Decide (deterministically or randomly) whether to add current round $t$ to $S$.
        \item For each round $s \in S$ whose delay expires this round (i.e., $s + d_s = t$), observe feedback from round $s$, and send this feedback to the learning algorithm.
        \item If the setting is \preemptive, decide which rounds, if any, to remove from $S$.
    \end{enumerate}
}}
\end{center}
We focus on randomized schedulers that operate autonomously from the learning algorithm, determining the observation schedule without relying on losses or actions. 
The following definitions formalize the aspects of the modularity that we will rely upon in our proofs.

\subsection{\Precom Schedulers and Observation-Independent Algorithms}
Let $S_t^0$ denote the state of the tracking set immediately before round $t$, while $S_t^1$ denotes the state after the decision whether to include round $t$ in the tracking set has been fully processed and before removing any elements. The observation indicator $Z_t = \ind(t \in S_{t+d_t}^1)$ denotes whether feedback from $t$ is observed at time $t+d_t$.  See \Cref{tab:notation-global} for a summary of the notation.

\begin{definition}
    A scheduler $\cS$ is \textbf{\precom}, relative to a filtration $(\cF_t^{\cS})$, if 
    there exists an i.i.d.\ sequence $(X_t^{\cS})_{t=1}^T$
    such that $\cF_t^{\cS} = \sigma(X_1^{\cS}, \dots, X_{t-1}^{\cS})$,
    each tracking set $S_t^0$ is $\cF_t^{\cS}$-measurable, and 
    each observation indicator $Z_t$ is $\cF_{t+1}^{\cS}$-measurable.  
\end{definition}
Hence, the tracking set at round $t$ is determined by randomness up to the start of round $t$, while feedback from round $t$ is observed at round $t + d_t$ based on the scheduler's randomness up to the start of the next round.

\begin{definition}
    A \precom scheduler $\cS$ is \textbf{quantified} by a sequence of non-zero probabilities $(p_t)_{t=1}^T$ if, for all $t\in[T]$, the observation indicator $Z_t$ satisfies $\E[Z_t \mid \cF_t^{\cS}] = p_t \ind(|S_t^0| < C)$.
\end{definition}
Hence, conditional on the tracking set not being full at the start of round $t$, feedback from round $t$ is observed with probability $p_t$.
\begin{definition}
\label{defn:nice}
    Let $\cA$ be a delay scheduling algorithm and let $\cS$ be its scheduler. We say $\cA$ is \textbf{\obsind} if $\cS$ is \precom relative to some filtration $(\cF_t^{\cS})$ and 
    there exists an independent i.i.d. sequence $(X_t^{\cA})_{t=1}^T$ 
    such that the action $A_t$ 
    at round $t$ is $\cF_t^{\cA}$-measurable,
    where $\cF_t^{\cA} = \sigma(\cF_t; X_t^{\cA})$ and $\cF_t = \sigma(\cF_t^{\cS};X_1^{\cA},\dots,X_{t-1}^{\cA})$.
\end{definition}
Hence, the learner does not influence the scheduler, and the learner at round $t$ only depends on the scheduler's randomness up to the start of that round, implying that
feedback for round $t$ is received at $t + d_t$ independently of $A_t$, i.e., $Z_t \perp A_t \mid \cF_t$.

\begin{algorithm}[H]
\caption{Delayed FTRL with access to a \precom scheduler $\cS$}\label{alg:FTRL-with-scheduler}
\SetKwInOut{Input}{Input}
\Input{Number of arms $K$, Capacity $C$.}
\SetKwInOut{Access}{Access}
\Access{\Precom scheduler $\cS$ quantified by a sequence $(p_t)_{t=1}^T$}
\begin{enumerate}[leftmargin=0.45cm,noitemsep,before=\vspace{-0.4cm},after=\vspace{-0.5cm}]
    \item Initialize scheduler $\cS$ together with the \textit{tracking set} $S$ of maximum size $C$. Initialize $\htLobs_{1} = \bzero_K$.
    \item \textbf{For} round $t = 1, 2, \ldots, T$:
    \begin{enumerate}[before=\vspace{-0.1cm}]
        \item Calculate learning rates $\alpha_t, \beta_t$ using available information. Draw an action $A_t$ according to\\ $x_t = \argmin_{x \in \Delta([K])} \pair{x, \htLobs_t} + \alpha_t^{-1} F_{\text{Ts}}(x) + \beta_t^{-1} F_{\text{NE}}(x)$ and play it.
        \item Scheduler $\cS$ makes a decision whether to start tracking round $t$ or not.
        \item For each round $s\in S$ whose delay expires this round (i.e., $s + d_s = t$), observe feedback from round $s$ (via scheduler $\cS$), calculate probability $p_s$, and  construct loss estimator $\htl_{s}$:
        \begin{itemize}[noitemsep, before=\vspace{-0.1cm}]
            \item In the bandit regime, observe $(s, l_{s,A_{s}})$, and set $\htl_{s} = l_{s,A_{s}} x_{s,A_{s}}^{-1} \be_{A_{s}} \cdot p_s^{-1}$.
            \item In the full-information regime, observe $(s, l_{s})$, and set $\htl_{s} = l_{s} \cdot p_s^{-1}$.
        \end{itemize}
        \item Update $\htLobs_{t+1} = \htLobs_{t} + \sum_{\text{observed } s:s+d_s=t} \htl_{s}$.
        \item Scheduler $\cS$ makes preemption decisions.
    \end{enumerate}
\end{enumerate}
\end{algorithm}

Algorithm~\ref{alg:FTRL-with-scheduler} is a delayed scheduling algorithm where the learner is Delayed FTRL and the \precom scheduler $\cS$ is quantified by a sequence $(p_t)_{t=1}^T$, which determines the weights of observations in the loss estimators.

\begin{theorem}\label{thm:ftrl-scheduler}
    Consider Algorithm~\ref{alg:FTRL-with-scheduler} as $\cA$, where a Delayed FTRL algorithm is run with access to a \precom scheduler $\cS$, quantified by a sequence $(p_t)_{t=1}^T$. If the learning rates $\alpha_t$ and $\beta_t$ are non-increasing and $\cF_{t}^{\cS}$-measurable, then $\cA$ is an \obsind delay scheduling algorithm whose regret, in the bandit regime, satisfies:
    \begin{equation*}
        \eR_T \le \E\sbr{\tsum_{t=1}^T \rbr{\sqrt{K}\alpha_t \frac{Z_t}{p_t^2} + \beta_t \frac{Z_t}{p_t} \tsum_{s\in\cW_t}\frac{Z_s}{p_s}} + 2\sqrt{K}\alpha_T^{-1} + \log(K)\beta_T^{-1}} + \tsum_{t=1}^T \Pr(|S_t^0| = C),
    \end{equation*}
    and in the full-information regime:
    \begin{equation*}
        \eR_T \le \E\sbr{\tsum_{t=1}^T \rbr{\beta_t \frac{Z_t}{p_t^2} + \beta_t \frac{Z_t}{p_t} \tsum_{s\in\cW_t}\frac{Z_s}{p_s}} + \log(K)\beta_T^{-1}} + \tsum_{t=1}^T \Pr(|S_t^0| = C).
    \end{equation*}
\end{theorem}
The proof of Theorem~\ref{thm:ftrl-scheduler} is in Appendix~\ref{app:ftrl-scheduler}. We show that $\cA$ is \obsind by constructing the sequence $(X_t^{\cA})_{t=1}^T$ via induction. In the regret analysis, we  forfeit rounds where $|S_t^0| = C$, incurring regret of $1$ per round, and use the fact that the scheduler and algorithm are \precom and \obsind respectively in order to condition on the scheduler's randomness, reducing the analysis to Delayed FTRL with loss scales $B_t = \frac{Z_t}{p_t}$, for which we apply Theorem~\ref{thm:FTRL-no-scale}.

\newpage
\subsection{\Precom Non-\Preemptive Bernoulli Scheduler}\label{subsec:precom-non-preempt}

From a practical standpoint, non-\preemptive schedulers are fairly straightforward as they only have to decide whether to start tracking a given round without managing preemption decisions. In that vein, the non-\preemptive \precom scheduler quantified by a sequence $(p_t)_{t=1}^T$ can only be implemented as the Bernoulli scheduler (Scheduler \ref{alg:scheduler-bernoulli}). This scheduler, independently decides at each round $t$ whether to track that round with probability $p_t$ if the tracking set is not full.

\begin{algorithm-scheduler}[H]
\caption{Non-\Preemptive Bernoulli Scheduler}\label{alg:scheduler-bernoulli}
\SetKwInOut{Inputs}{Inputs}
\Inputs{Empty tracking set $S$ of maximum size $C$.}
\textbf{For} round $t = 1, 2, \ldots, T$:
\begin{enumerate}[leftmargin=1cm,noitemsep,before=\vspace{-0.35cm},after=\vspace{-0.3cm}]
    \item Calculate $p_t$ using available information. If $|S| < C$, then add $t$ to $S$ with probability $p_t$.
    \item For each expired $s \in S$ (i.e., $s + d_s = t$), observe its feedback, and pass it to the learner.
\end{enumerate}
\end{algorithm-scheduler}

For Scheduler~\ref{alg:scheduler-bernoulli}, $p_t$ only needs to be deterministically computable\footnote{“Deterministically computable” means that the value remains the same in every run, computed without any randomness.} at each round $t$. By carefully selecting these probabilities, we can ensure representative feedback is captured throughout the game while keeping the probability of reaching the full capacity~$C$ at any round under control.

Under clairvoyance, one viable sequence class is given by $p_t = \min\{1, \frac{C}{(1+\alpha)\nu_t}\cdot \frac{1}{d_t+1}\}$, governed by two components: a normalizer sequence $(\nu_t)_{t=1}^T$ and a tunable Chernoff parameter $\alpha > 0$. The sequence $\nu_t$ helps to control the expected tracking set size $\E[|S_t^0|]$, while a sufficiently large $\alpha$ ensures a high probability bound on $|S_t^0|$. We set $\nu_t = 2H_t$, where $H_t = \sum_{t=1}^T \frac{1}{t}$ is a harmonic number. The value of $\alpha$ is taken large enough to satisfy the following ``Overflow condition":

\begin{equation}
    \ln(1+\alpha) - \tfrac{\alpha}{1+\alpha} \ge \tfrac{\ln(\delta^{-1})}{C},\label{eq:chernoff-condition}\tag{$\star$}
\end{equation}

for some overflow probability $\delta \in (0,1)$. These choices are used in Lemma~\ref{lem:overflow-non-preemptive} and later Lemma~\ref{lem:overflow-bound}.

For this choice of $p_t$, Lemma~\ref{lem:overflow-non-preemptive} provides a high-probability guarantee that the tracking set is not full at any round. The proof, given in Appendix~\ref{app:proxy-delays}, follows from a multiplicative Chernoff bound.

\begin{lemma}[Bernoulli Scheduler: Capacity Control]\label{lem:overflow-non-preemptive}
    Let $\delta \in (0, 1)$. Suppose Chernoff parameter $\alpha>0$ satisfies \eqref{eq:chernoff-condition}, then Scheduler \ref{alg:scheduler-bernoulli} with probabilities $p_t = \min\{1, \frac{C}{2 (1+\alpha) H_t}\cdot \frac{1}{d_t+1}\}$ guarantees that $\Pr(|S_t^0| = C) \le \delta$ for every $t\in [T]$.
\end{lemma}

\subsection{\Precom \Preemptive Scheduling via Proxy Delays}

We introduce \textbf{proxy delays}, a novel approach to \preemptive scheduling in which, at the beginning of each round $t$, the scheduler selects a proxy delay $\wtd_t \in \Z_{\ge -1}$ independently of previous rounds, determining how long round $t$ will be tracked. Thus, if the tracking set is not full and $\wtd_t \ge 0$\footnote{A proxy delay of $\wtd_t = -1$ indicates that round $t$ is never added to the tracking set $S$.}, round $t$ remains in $S$ until the end of round $t + \min\{d_t, \wtd_t\}$. This ensures tracking durations are determined at round $t$, keeping the scheduler \precom, as $Z_t = \ind(|S_t^0| < C, \wtd_t \ge d_t)$. Proxy delays are effective in non-clairvoyant settings, as they can be sampled without knowledge of $d_t$.

A natural choice for the proxy delay distribution is the \emph{Pareto distribution}, whose heavy tail ensures long delays are observed with non-negligible probability, while its inverse-polynomial decay helps control capacity. Formally, for the \emph{scale} and \emph{shape} parameters $c,\beta>0$, the CDF is defined as $F(x) = \ind(x \ge c)(1-(\frac{c}{x})^\beta)$.

Setting the shape parameter $\beta=1$ ensures that the probability of observing feedback from a round with delay $d$ is proportional to $1/d$, aligning with the optimal sampling rate for fixed delays in Theorem~\ref{thm:fixed-delays}. 

The full scheduling policy using Pareto-distributed proxy delays is given in Scheduler~\ref{alg:scheduler-proxy-delays}. Here, the scheduler samples $\wtd_t$ from the distribution $\gD_t = \lfloor \text{Pareto}(\tfrac{C}{(1+\alpha)\nu_t}, 1) - 1 \rfloor$, where the normalizer sequence $\nu_t = 2H_t$ and Chernoff parameter $\alpha$ are defined analogously to those in Subsection~\ref{subsec:precom-non-preempt}. This parameter choice ensures that Scheduler~\ref{alg:scheduler-proxy-delays} is quantified by the same probability sequence as Scheduler~\ref{alg:scheduler-bernoulli} in Lemma~\ref{lem:overflow-non-preemptive}, whose preemptive analogue is covered in Lemma~\ref{lem:overflow-bound} below.

\begin{algorithm-scheduler}[htbp]
\caption{\Preemptive Scheduler with Pareto Proxy Delays}\label{alg:scheduler-proxy-delays}
\SetKwInOut{Inputs}{Inputs}
\Inputs{Empty tracking set $S$ of maximum size $C$.}
\SetKwInOut{Parameters}{Parameters}
\Parameters{Chernoff parameter $\alpha$.}
\textbf{For} round $t = 1, 2, \ldots, T$:
\begin{enumerate}[leftmargin=1cm,noitemsep,before=\vspace{-0.35cm},after=\vspace{-0.3cm}]
    \item Sample proxy delay $\wtd_t \sim \gD_t$. Add round $t$ to $S$ if $|S| < C$ and $\wtd_t \ge 0$.
    \item For each expired $s \in S$ (i.e., $s + d_s = t$), observe its feedback, and pass it to the learner.
    \item For each $s \in S$ whose proxy delay expires this round (i.e., $s + \wtd_s = t$), remove $s$ from $S$.
\end{enumerate}
\end{algorithm-scheduler}

\begin{lemma}[Proxy Delay Scheduler: Observation Probability and Capacity Control]\label{lem:overflow-bound}
    Let $\delta \in (0,1)$. If the Chernoff parameter $\alpha$ satisfies \eqref{eq:chernoff-condition}, then Scheduler \ref{alg:scheduler-proxy-delays} ensures that $\E[Z_t\mid |S_t^0| < C] = \Pr(\wtd_t \ge d_t) = \min\{1, \frac{C}{2 (1+\alpha) H_t}\cdot \frac{1}{d_t+1}\}$ and $\Pr(|S_t^0| = C) \le \delta$ for every $t\in [T]$.
\end{lemma}

By construction, Scheduler~\ref{alg:scheduler-proxy-delays} ensures that $Z_t = \ind(|S_t^0| < C, \wtd_t \ge d_t)$, yielding the first result by evaluating tails of $\gD_t$. To analyze the capacity constraint, we let $\wtsig_t = \sum_{s=1}^{t-1} \ind(s + \wtd_s \ge t)$ denote the number of outstanding proxy delays at round $t$, which bounds the tracking set size at the start of round $t$. The normalizer sequence $\nu_t$ controls $\E[\wtsig_t]$, while condition~\eqref{eq:chernoff-condition} guarantees that $\alpha$ is large enough to obtain a strong multiplicative Chernoff bound on $\Pr(\wtsig_t \ge C)$. The proof of Lemma~\ref{lem:overflow-bound} is in Appendix~\ref{app:proxy-delays}.

The illustration below (\Cref{fig:proxy-delays}) shows how the tracking set evolves in a single sample run.

\begin{figure}[H]
    \centering
    \includegraphics[width=0.7\linewidth]{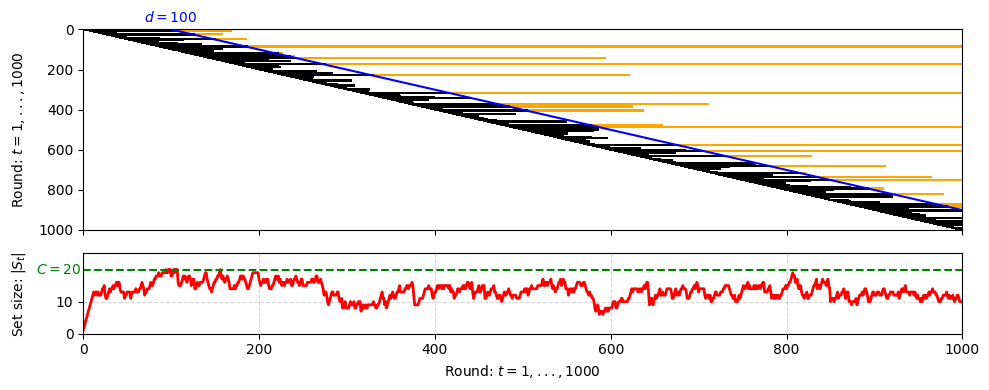}
    \caption{Example of the evolution of the tracking set: $T = 1000$, $C = 20$, $d = 100$. Black intervals $[t,\, t + \min\{d,\, \wtd_t\}]$ indicate when round $t$ is in the tracking set, while orange intervals $[t + d,\, t + \wtd_t]$ represent the excess proxy delay for rounds where feedback is observed.}
    \label{fig:proxy-delays}
\end{figure}

\section{Upper Bounds for Clairvoyant or \Preemptive Settings}\label{sec:upper-bounds-CNP-NCP}
In this section, we apply the techniques developed in \Cref{sec:scheduling-and-learning} to conclude the regret bounds for Clairvoyant Non-\Preemptive and Non-Clairvoyant \Preemptive settings.
In order to prove the corollaries below, we apply Theorem~\ref{thm:ftrl-scheduler} (Delayed FTRL with \precom schedulers) with Scheduler \ref{alg:scheduler-bernoulli} for \Cref{thm:CNP-results}, and with Scheduler  \ref{alg:scheduler-proxy-delays} for \Cref{thm:NCP-results}, and choosing learning rates that are $\cF^{\cS}_t$-measurable and computable from the information available at each round $t$. 
(See Appendix~\ref{app:NCP-CNP-results} for proofs.) 
We use $\mu_t = (\Pr(\wtd_t \ge d_t))^{-1} = \max\{1, \frac{(1+\alpha)\nu_{t}}{C}\cdot (d_t + 1)\}$ and $z_t = Z_t \mu_t$ for tuning the learning rates in the following corollaries. 
\begin{theorem}[Clairvoyant and Non-\Preemptive]\label{thm:CNP-results}
    Let $\cS$ be Scheduler~\ref{alg:scheduler-bernoulli} with probabilities $p_t = \min\{1, \frac{C}{(1+\alpha)\nu_t}\cdot \frac{1}{d_t+1}\}$ such that $\alpha$ satisfies \eqref{eq:chernoff-condition} for some $\delta \in (0,1)$. Then, in the bandit regime, Algorithm~\ref{alg:FTRL-with-scheduler} with access to $\cS$ has the following expected regret bound, when run with learning rates $\alpha_t = \sqrt{\frac{1}{\sum_{s\in [t]} \mu_s}}, \beta_t = \sqrt{\frac{\log(K)}{\sum_{s\in [t]} d_s}}$: 
    \begin{align*}
        \eR_T \le 
        4\sqrt{K} \sqrt{T + \tfrac{(1+\alpha)\nu_{T}}{C} (D + T)} + 3\sqrt{D \log(K)} + \delta T.
    \end{align*}
    And in the full-information regime, with learning rate $\beta_t = \sqrt{\frac{\log(K)}{\sum_{s\in [t]} (\mu_s + d_s)}}$: $$ \eR_T \le 3\sqrt{\log(K)}\sqrt{(D + T)(1+\tfrac{(1+\alpha)\nu_T}{C})} + \delta T.$$    
\end{theorem} 

For each round $t$, let $O_t = \{s \in [t-1] : s + d_s < t\}$ denote the \emph{observation set} of rounds whose feedback might be available. In Delay Scheduling, the player only observes its subset $\widetilde O_t = \{s \in O_t: Z_s = 1\}$ by the start of round $t$. Note that for all $s\in O_t\setminus \widetilde O_t$, $Z_s = z_s = 0$.

\begin{theorem}[Non-Clairvoyant and \Preemptive with known $\dmax$]\label{thm:NCP-results}
    Let $\mu_{\max, t} = \max\{1, \frac{(1+\alpha)\nu_{t}}{C}\cdot (\dmax + 1)\}$ and $\mu_{\max} = \mu_{\max,T}$. Let $\cS$ be Scheduler~\ref{alg:scheduler-proxy-delays} with parameter $\alpha$ satisfying \eqref{eq:chernoff-condition} for some $\delta \in (0,1)$. Then, in the bandit regime, Algorithm~\ref{alg:FTRL-with-scheduler} with access to $\cS$ has the following regret bound, when run with learning rates $\alpha_t =  \sqrt{\frac{1}{\sum_{s\in \wtO_t} z_s^2 + C \mu_{\max, t}^2}}, \beta_t = \sqrt{\frac{\log(K)}{\sum_{s\in \wtO_t} z_s d_s + C \mu_{\max, t}\dmax}}$:
    {
    \begin{align*}
        \eR_T 
        \le 4\sqrt{K} \sqrt{T + \tfrac{(1+\alpha)\nu_{T}}{C} (D + T)} 
        + 3\sqrt{D\log(K)}
        + 7\sqrt{C \mu_{\max}(K\mu_{\max} + \log(K)\dmax)} + \delta T.
    \end{align*}
    }
    And in the full-information regime, with learning rate $\beta_t = \sqrt{\frac{\log(K)}{\sum_{s\in \wtO_t} z_s(z_s + d_s) + C \mu_{\max,t}(\mu_{\max,t} + \dmax)}}$:
    {
    \begin{align*}
        \eR_T 
        \le 3\sqrt{\log(K)} \sqrt{(D + T)(1+\tfrac{(1+\alpha)\nu_T}{C})} 
        + 3\sqrt{C\mu_{\max}\log(K)(\mu_{\max} + \dmax)} + \delta T.
    \end{align*}
    }
\end{theorem}

\paragraph{Parameter Choices and Result Summary:} To provide context for these results, we need to set parameter $\alpha$ large enough so that \eqref{eq:chernoff-condition} is satisfied for some $\delta \in (0,1)$. Assuming $C\ge 3\log(T)$, we can set $\alpha = 1$ without prior knowledge of $T$, ensuring that \eqref{eq:chernoff-condition} holds for $\delta = T^{-0.5}$. In this case, ~\Cref{thm:CNP-results,thm:NCP-results} yield the results presented in Tables~\ref{tab:results-CNP-general} and \ref{tab:results-NCP-general}, respectively. 

More generally, for all $C\ge 1$, if we were to set $\alpha = e T^{0.5/C} - 1$ and consider $\delta = T^{-0.5}$, then we would have the regret bounds summarized in Table \ref{tab:results-general}.

\begin{table}[H] \centering 
\resizebox{1\columnwidth}{!}{
\begin{tabular}{c|c|c} 
\hline
Framework & Regime & Regret Bounds\\ 
\hline
\multirow{2}{*}{\makecell{Clairvoyant\\ Non-\Preemptive}} 
& Bandit & $O\rbr{\sqrt{TK + \tfrac{T^{0.5/C}\log(T)}{C}(D+T)K} + \sqrt{D\log(K)}}$\\
& Full-info & $O\rbr{\sqrt{(1+\tfrac{T^{0.5/C}\log(T)}{C})(D + T)\log(K)}}$\\
\hline
\multirow{2}{*}{\makecell{Non-Clairvoyant\\ \Preemptive}} 
& Bandit & $O\rbr{\sqrt{TK + \tfrac{T^{0.5/C}\log(T)}{C}(D+T)K} + \sqrt{D\log(K)}} + \widetilde O\rbr{\dmax\sqrt{T^{0.5/C} + \tfrac{T^{1/C}K}{C}}}$\\
& Full-info & $O\rbr{\sqrt{(1+\tfrac{T^{0.5/C}\log(T)}{C})(D + T)\log(K)}} + \widetilde O\rbr{\dmax\sqrt{T^{0.5/C} + \tfrac{T^{1/C}}{C}}}$\\
\hline
\end{tabular} 
}
\caption{Regret upper bounds derived from Theorems~\ref{thm:CNP-results} and \ref{thm:NCP-results}, provided we set $\alpha = eT^{0.5/C} - 1$. 
}   
\label{tab:results-general}
\end{table}

When $C = \Omega(\log T)$, we have $T^{1/C} = O(1)$ and the bounds match those in Tables~\ref{tab:results-CNP-general} and \ref{tab:results-NCP-general}.

\section{Discussion and Future Work}

We introduce the Delay Scheduling setting, a novel framework for online learning with delayed feedback under a capacity constraint. Our analysis spans various settings characterized by distinct delay structures, clairvoyance, and preemptibility, covering both bandit and full-information feedback regimes. A key finding reveals that, in many cases, remarkably modest capacity suffices to achieve regret comparable to that of unconstrained Delayed Online Learning. Building on these results, we identify several critical directions for future inquiry.

A key open question in our work is determining the minimax regret when $C=O(\log T)$. Although we establish minimax regret bounds for $C=\Omega(\log T)$ under clairvoyance or preemptibility, the precise dependence of regret on $C$ in the small-capacity regime remains unclear. Tightening these bounds would refine our understanding of the fundamental capacity requirements for efficient learning. Moreover, designing algorithms for Non-Clairvoyant, Non-Preemptive Delay Scheduling that require no prior knowledge of $T$, $D$, or $\dmax$ represents a significant challenge. Additionally, since existing lower bounds are derived via reductions from Delay Scheduling with fixed delays, establishing setting-specific lower bounds remains an important open problem.

Recent work has explored delayed bandits under various assumptions that bring the setting closer to real-world applications. It is promising to extend these frameworks by incorporating the capacity constraint introduced in this paper. For instance, Delay Scheduling with action-dependent delays, where delay duration varies based on chosen actions in each round, as in \cite{vanderhoeven2022}, represents a promising research direction. Such action-dependent delays naturally occur in applications like dynamic pricing and medical trials, necessitating novel algorithmic approaches, particularly in non-clairvoyant settings. Investigating Delay Scheduling with scale-free losses, where delay correlates with the incurred loss, as in \cite{huang2023bankeronlinemirrordescent}, offers another exciting extension. If longer delays correspond to systematically different loss distributions, dynamically adjusting sampling rates could enhance performance. Finally, contextual bandits with delayed feedback, as in \cite{yishay2024contextualbandits}, considered under the capacity constraint, present an important open direction for future work. In applications such as personalized recommendations, where context-action pairs must be tracked, capacity constraint introduces new challenges for exploration-exploitation trade-offs.

Overall, our results establish a foundation for learning with capacity-constrained delayed feedback, suggesting many promising directions for future research.

\section{Acknowledgments and Funding}

We would like to thank Liad Erez and Tal Lancewicki for valuable discussions on online learning with delays, and Allan Borodin for insightful conversations on job scheduling problems.

Idan Attias is supported by the National Science Foundation under Grant ECCS-2217023, through the Institute for Data, Econometrics, Algorithms, and Learning (IDEAL). Daniel M. Roy is supported by the funding through NSERC Discovery Grant and Canada CIFAR AI Chair at the Vector Institute.

\printbibliography

\newpage
\appendix

\section{Additional Related Work}\label{app:additional-related-work}

\paragraph{Online learning with delays.}

One of the earliest works in online learning with delays against an oblivious adversary was by \cite{weinberger2002}, where they determined minimax regret of $\Theta(\sqrt{T(d+1)\log(K)})$ for the full-information feedback under fixed delays. For the full-information regime with round-dependent delays, \cite{joulani2016, quanrud2015} proved the minimax regret bound of $\Theta(\sqrt{(D+T)\log(K)})$. The delayed bandit setting was later introduced by \cite{cesa16}, who established a regret lower bound of $\Omega(\max\{\sqrt{TK}, \sqrt{Td\log(K)}\})$ for the fixed delay case, where delays are fixed as $d_t = d$. For the round-dependent delay case, \cite{thune} proposed the Delayed Exponential Weights algorithm, achieving an almost matching regret bound of $O(\sqrt{(TK + D)\log(K)})$. In the following year, \cite{seldin20} proposed an FTRL-based algorithm, which had $O(\sqrt{TK + D\log(K)})$ regret, matching the lower bound from \cite{cesa16}.

There has been ongoing interest in online and bandit convex optimization under delayed feedback, namely \cite{quanrud2015, joulani2016, heliou2020gradientfreeonlinelearninggames, bistritz22}, and more recently \cite{wan2024improvedregretbanditconvex, qui2025exploitingcurvatureonlineconvex}. Notably, \cite{wan2024improvedregretbanditconvex} proposed an algorithm that also employs a batching technique similar in spirit to the one used in our Algorithm~\ref{alg:batch-partition}.

Significant attention has also been given to the stochastic version of the setting, notably by \cite{joulani2013} and \cite{lancewicki21a}. Furthermore, the FTRL-based approach from \cite{seldin20} has been extended to the best-of-both-worlds framework in \cite{masoudian2022bestofbothworldsalgorithmbanditsdelayed, masoudian2024bestofbothworldsalgorithmbanditsdelayed}.

Recently, a number of studies have explored delayed bandits under assumptions that mirror real-life applications. For instance, \cite{yishay2024contextualbandits} examined contextual bandits with delayed feedback, \cite{vanderhoeven2022} considered bandits with arm-dependent delays, and \cite{esposito2023delayed} investigated delayed bandits with intermediate observations.

We further explore the connection of our Delay Scheduling to Delayed Online Learning in Appendix~\ref{app:le-learning}.

\paragraph{Label-efficient online learning.}
In a label-efficient game, as proposed by \cite{helmbold1997}, it is assumed that the learner can query feedback from at most $M$ rounds out of $T$. The full-information setting has been shown to have minimax regret of $\Theta\trbr{T\sqrt{\log(K) / M}}$ by \cite{cesa2005}. Next, \cite{audibert10} have solved the bandit feedback regime with minimax regret of $\Theta\trbr{T\sqrt{K / M}}$. 

Appendix~\ref{app:delayed-online-learning} provides additional insights into the relationship between our Delay Scheduling and label-efficient online learning.

\paragraph{Online job scheduling.}
Online Job Scheduling (e.g., see \cite{borodin1998, pinedo2022scheduling}) is a broad research area that involves studying the problem of assigning sequentially arriving jobs across multiple resources with the goal of optimizing specific objectives, such as minimizing maximum tardiness or maximizing weighted throughput. Online Interval Scheduling is a variant of Online Job Scheduling  where jobs have fixed starting and end times, which is precisely our case in Delay Scheduling. In \textit{Online Interval Scheduling} with multiple resources, the algorithm is presented with a sequence of time intervals ordered by their starting times, and it must immediately assign each interval to one of the machines or reject it, ensuring no overlap on the same machine while optimizing the overall schedule with respect to some metric (e.g., number of intervals, total length). In the original paper by \cite{lipton1994}, intervals could not be unscheduled and their lengths were announced at starting times. \cite{woeginger1994} considered a modification of this setting where intervals can be unscheduled.

In the analogy between our Delay Scheduling setting and Online Job (Interval) Scheduling, rounds correspond to jobs with increasing arrival times $(\mathsf{a}_t)$ and arbitrary processing times $(\mathsf{p}_t)$, where capacity $C$ represents the number of initially idle resources available for processing jobs. A job can only be assigned to an idle resource upon arrival, after which the resource becomes busy and processing continues until time $\mathsf{a}_t + \mathsf{p}_t$ or possible preemption, when resource becomes idle again. Then, the delay $d_t$ of round $t$ corresponds to the number of jobs arriving during the processing interval $[\mathsf{a}_t, \mathsf{a}_t + \mathsf{p}_t)$ of job $t$. Note that long processing time does not necessarily imply long delay, as delays serve as the measure of how concurrent the jobs are.

\paragraph{Streaming.} Streaming algorithms are specialized algorithms designed to process massive data streams using limited memory to answer queries approximately. Instead of storing and processing the entire dataset, these algorithms make a single pass (or very few passes) over the data, maintaining a small summary, usually referred to as a sketch,  which captures essential information about the stream.
By analogy, the limitation on space in streaming is replaced in our model by the constraint on how many rounds should be tracked in the online learning with delays model. In both models, a key technique is to use randomized sampling to select elements from the stream in a way that still allows for good performance: answering queries in streaming and minimizing regret in our setting.

 The foundational work of Alon, Matias, and Szegedy \cite{alon1996space} introduced the streaming model and developed space-efficient algorithms for estimating frequency moments. Cormode and Muthukrishnan \cite{cormode2005improved} proposed the Count-Min Sketch, a widely used tool for approximate frequency estimation. In the domain of cardinality estimation, Flajolet et al. \cite{flajolet2007hyperloglog} developed HyperLogLog, an optimal algorithm for counting distinct elements in a stream. This field continues to be extensively studied; see, for example, \cite{charikar2002finding,cormode2005s,datar2002maintaining,flajolet1985probabilistic,indyk2005optimal,kane2010optimal,manku2002approximate,muthukrishnan2005data}.

\subsection{Related Setting: Delayed Online Learning}\label{app:delayed-online-learning}

Delayed Online Learning considers scenarios with delayed feedback but no resource-driven constraints (i.e., $C = \infty$). It is straightforward to see that, for every selection of delays $\{d_t\}_{t=1}^T$, the regret in the Delay Scheduling game cannot exceed that of the corresponding Delayed Online Learning game with the same delays. Theorem~\ref{thm:cesa16-lb} presents the lower bound for the fixed-delay case in the bandit regime.

\begin{center}\resizebox{0.95\columnwidth}{!}{\fbox{\parbox{1\textwidth}{
    \textbf{Delayed Game}\vspace{5pt}\\
    $\bullet$ \textit{Visible Parameters:} number of actions $K$.\\
    $\bullet$ \textit{Latent Parameters:} number of rounds $T$.\\
    $\bullet$ \textit{Pre-game:} adversary selects losses $l_{t} \in [0, 1]^K$ and  delays $d_t \in [T-t]$ for all $t \in [T]$.
    
    \vspace{5pt}
    \noindent For each round $t=1, 2, \ldots, T$ repeat:
    \vspace{-7pt}
    \begin{enumerate}[noitemsep]
        \item[0.] If the setting is \textbf{clairvoyant}, then the environment reveals $d_t$.  
        \item The player plays $A_t\in [K]$ and incurs corresponding loss $l_{t,A_t}$. 
        \item For all $s \le t$ such that $s+d_s = t$, the environment reveals:
        \begin{itemize}[noitemsep,after=\vspace{-0.3cm},before=\vspace{-0.1cm}]
            \item index-value pair $(s, l_{s,A_s})$ in the \textbf{multi-armed bandit} game,
            \item index-vector pair $(s, l_s)$ in the \textbf{full-information} game.
        \end{itemize}
    \end{enumerate}
}}}\end{center}

\begin{theorem}[\cite{cesa16}, proof of Corollary 11, Appendix D]\label{thm:cesa16-lb} 
The minimax regret in the multi-armed bandit setting with fixed delays $d_t = d$ is of the order
\begin{align*}
    \Omega\rbr{\max\calbr{\sqrt{KT}, \sqrt{Td\log(K)}}}.
\end{align*}
\end{theorem}

\subsection{Related Setting: Label Efficient Online Learning}\label{app:le-learning}

Since the total sum of delays for observed rounds in delay scheduling is bounded by $CT$, whereas the total delay $D$ can grow quadratically in terms of $T$ in the worst case, the capacity constraint may significantly limit the number of observations. Previously studied label efficient games have already explored scenarios where the number of observations is bounded.

\begin{center}\fbox{\parbox{0.95\textwidth}{
    \textbf{Label Efficient Game}\vspace{5pt}\\
    $\bullet$ \textit{Visible Parameters:} number of arms $K$, number of rounds $T$, number of queries $M$.\\
    $\bullet$ \textit{Pre-game:} adversary selects losses $l_{t,i} \in [0, 1]$ for all $t \in [T]$ and $i \in [K]$.
    
    \vspace{5pt}
    \noindent For each round $t=1, 2, \ldots, T$ repeat:
    \begin{enumerate}[leftmargin=1cm,noitemsep,before=\vspace{-0.3cm},after=\vspace{-0.3cm}]
        \item The player plays $A_t\in [K]$ and incurs corresponding loss $l_{t,A_t}$. 
        \item The player observes:
        \begin{itemize}[noitemsep,before=\vspace{-0.1cm},after=\vspace{-0.1cm}]
            \item value $l_{s,A_s}$ in the \textbf{multi-armed bandit} game,
            \item vector $l_s$ in the \textbf{full-information} game,
        \end{itemize}
        only if he asks for it with the global constraint that he is not allowed to ask it more than $M$ times throughout the game.
    \end{enumerate}
}}
\end{center}

In particular, it has been shown that when the number of observations is capped at $M$, the minimax regret of a label efficient game is $\Theta(T\sqrt{K/M})$ in the bandit regime and $\Theta(T\sqrt{\log(K)/M})$ in the full-information regime. It follows directly that if, for some selection of delays $\{d_t\}_{t=1}^T$ in delay scheduling, it is impossible to make more than $M$ observations without violating the capacity constraint, then the regret in the delay scheduling game cannot exceed that of the corresponding label efficient game with $M$ queries.

\begin{theorem}[\cite{audibert10}, Theorem 30]\label{thm:audibert10-lb}
Let $M > K$. Consider a label efficient game where a player can query feedback from at most $M$ rounds. Let $\sup$ be taken over all oblivious adversaries and $\inf$ over all players, then the following holds true in the label efficient full-information game:
\begin{equation*} 
    \inf\sup \eR_T \ge 0.03 T \sqrt{\tfrac{\log(K)}{M}},    
\end{equation*}
and in the label efficient bandit game we have:
\begin{equation*} 
    \inf\sup \eR_T \ge 0.04 T \sqrt{\tfrac{K}{M}}.    
\end{equation*}
\end{theorem}

\section{Additional Preliminaries}\label{app:additional-prelim}

For the reader's convenience, we provide a consolidated summary of the notation used throughout the paper. Table~\ref{tab:notation-global} lists the key symbols along with their definitions. Additionally, we include the definition of the Pareto distribution (Definition~\ref{def:pareto}) and a brief overview of harmonic numbers.

\begin{table}[H]
\centering
    \renewcommand{\arraystretch}{0.4} 
\begin{tabular}{ll}
\toprule
\textbf{Symbol} & \textbf{Definition / Description}\\
\midrule
$L_{t,a}$ & $\sum_{s=1}^{t-1} l_{t,a}$ \quad (Cumulative loss for action $a$ up to time $t$)\\[6pt]
$\eR_{T,a}$ & $\E\sbr{\sum_{t=1}^T l_{t,A_t}} - L_{T+1,a}$ \quad (Expected regret w.r.t.\ action $a$)\\[6pt]
$i^*$ & $\argmin_{a \in [K]} \{ L_{T+1,a}\}$ \quad (Best action in hindsight)\\[6pt]
$\eR_T$ & $\eR_{T, i^*}$ \quad (Expected regret w.r.t.\ the best action $i^*$)\\[6pt]

\midrule
$O_t$ & $\{ s \in [t-1] : s + d_s < t \}$ \quad (Observed set of round $t$)\\[6pt]
$\cW_t$ & $\{ s \in [t-1] : s + d_s \ge t\}$ \quad (Working set of round $t$)\\[6pt]

$\sigma_t$ & $|\cW_t|$ \quad (Number of outstanding delays at round $t$)\\[6pt]
$D$ & $\sum_{t=1}^T \sigma_t = \sum_{t=1}^T d_t$ \quad (Total delay across all rounds)\\[6pt]
$\sigmax$ & $\max_{t\in[T]} \sigma_t$ \quad (Maximum number of outstanding delays across all rounds)\\[6pt]
$\dmax$ & $\max_{t\in[T]} d_t$ \quad (Maximum delay across all rounds)\\[6pt]

\midrule
$S_t^0$ & State of the tracking set $S$ immediately before round $t$.\\[6pt]
$S_t^1$ & State of the tracking set $S$ in round $t$, after the decision whether to include \\ & $t$ in $S$ has been fully processed, and before removing any elements.\\[6pt]
$Z_t$ & $\ind\rbr{t \in S_{t+d_t}^1}$ \quad (Indicator that feedback from round $t$ is observed)\\
\bottomrule
\end{tabular}
\caption{Notation used throughout the paper.}
\label{tab:notation-global}
\end{table}

\begin{definition}[Pareto Distribution]
\label{def:pareto}
A random variable $X$ follows a \textit{Pareto} distribution with scale parameter $c > 0$ and shape parameter $\beta > 0$, denoted by $X \sim \mathrm{Pareto}(c,\beta)$, if its cumulative distribution function is given by $F_X(x) = \ind(x \ge c)(1-(\frac{c}{x})^\beta)$. 
\end{definition}

Let $H_t = \sum_{s=1}^t 1/s$ denote the $t$-th harmonic number. It is a well-known fact that $H_t = \log(t) + O(1)$, with $\gamma = \lim_{t\to \infty} (H_t - \log(t)) \approx 0.577$ known as the Euler--Mascheroni constant.

\section{Delayed FTRL with Time-Varying Loss Scales: Proof}\label{app:delayed-ftrl-proof}

In this section, we prove Theorem~\ref{thm:FTRL-no-scale}. Algorithm~\ref{alg:FTRL-no-scale} presents the Delayed FTRL algorithm, explored in Section~\ref{subsec:FTRL-no-scale}, with learning rate sequences taken as parameters. In the full-information regime, the algorithm uses the full loss vector as the estimator, while in the bandit regime, it performs importance weighting for the observed loss.

\begin{algorithm}[htbp]
\caption{Generic Delayed FTRL}\label{alg:FTRL-no-scale}
\SetKwInOut{Input}{Input}
\Input{Number of arms $K$.}
\SetKwInOut{Parameters}{Parameters}
\Parameters{Learning rates $(\alpha_t)_{t=1}^T$ and $(\beta_t)_{t=1}^T$.}
\begin{enumerate}[leftmargin=0.45cm,noitemsep,before=\vspace{-0.35cm},after=\vspace{-0.5cm}]
    \item Initialize $\htLobs_{1} = \bzero_K$.
    \item \textbf{For} round $t = 1, 2, \ldots, T$:
    \begin{enumerate}[before=\vspace{-0.1cm}]
        \item Draw and play an action $A_t \sim x_t = \argmin_{x \in \Delta([K])} \pair{x, \htLobs_t} + \alpha_t^{-1} F_{\text{Ts}}(x) + \beta_t^{-1} F_{\text{NE}}(x)$.
        \item For each round $s\in [t]$ whose delays expires this round (i.e., $s + d_s = t$):
        \begin{itemize}[noitemsep, before=\vspace{-0.1cm}, after=\vspace{-0.0cm}]
            \item \textbf{Bandit:} Observe $(s, l_{s,A_s})$ and construct estimator $\htl_s = l_{s,A_s} x_{s,A_s}^{-1} \be_{A_s}$.
            \item \textbf{Full-information:} Observe $(s, l_s)$ and construct estimator $\htl_s = l_s$.
        \end{itemize}
        \item Update $\htLobs_{t+1} = \htLobs_{t} + \sum_{s:s+d_s=t} \htl_{s}$.
    \end{enumerate}
\end{enumerate}
\end{algorithm}
Following \cite{seldin20}, the proof is structured into six facts and three lemmas. Before proceeding, we restate the notation from \cite{seldin20} related to the algorithm's regularization structure. Given non-increasing sequences of learning rates $\alpha_t, \beta_t$, the hybrid regularizer in round $t$ is $F_t(x) =  F_{t, 1}(x) + F_{t,2}(x)$, where $F_{t, 1}(x) = \alpha_t^{-1} F_{\text{Ts}}(x)$ and $ F_{t,2}(x) = \beta_t^{-1} F_{\text{NE}}(x)$ denote the components. For each regularizer $F\in \calbr{F_t}_{t\in[T]}$, we define unconstrained and constrained convex conjugates:
\begin{align*}
    &\conjF(\theta) = \sup_{x\in \R^k} \calbr{\pair{x, \theta} - F(x)},\\
    &\cconjF(\theta) = \sup_{x\in \Delta([K])} \calbr{\pair{x, \theta} - F(x)}.
\end{align*}
Define $f_t:\R \to \R$ as $f_t(x) = -2\alpha_t^{-1}\sqrt{x} + \beta_t^{-1} x\log(x)$, decomposed as $f_t = f_{t,1} + f_{t,2}$ with $f_{t,1}(x) = -2\alpha_t^{-1}\sqrt{x}$ and $f_{t,2}(x) = \beta_t^{-1} x\log(x)$. Then $F_t(x) = \sum_{i=1}^K f_t(x_i)$. The algorithm's update rule for the weights can be written as $x_t = \nabla \cconjF_t(-\htLobs_t)$ (e.g., see Theorems 5.7 and 6.8 from \cite{orabona_book}). Also, as $F_t^*$ considers maximization over $\R^K$, it holds that $F_t^*(\theta) = \sum_{i=1}^K f_t^*(\theta_i)$. 

\begin{proof}(Theorem~\ref{thm:FTRL-no-scale})
    Let $\htL_{t+1} = \sum_{s = 1}^{t} \htl_s$ for each $t\in [T]$. Let $i^* = \argmin_{i \in [K]} L_{t, i}$. Expand $\eR_{T}$ as follows
    \begin{align*}
        \eR_{T}
        &= \E\sbr{\tsum_{t=1}^T l_{t,A_t}} - L_{T+1, i^*}\\
        &= \E\sbr{\tsum_{t=1}^T \rbr{\cconjF_t(-\htLobs_t - \htl_t) - \cconjF_t(-\htLobs_t) + \pair{x_t, \htl_t}}}\\
        &+ \E\sbr{\tsum_{t=1}^T \rbr{\cconjF_t (-\htL_t) - \cconjF_t (-\htL_{t+1})} - \htL_{T+1,i^*}}\\
        &+ \E\sbr{\tsum_{t=1}^T \rbr{\cconjF_t(-\htLobs_t) - \cconjF_t(-\htLobs_t - \htl_t) - \cconjF_t (-\htL_t) + \cconjF_t (-\htL_{t+1})}}.
    \end{align*}
    The resulting three terms can be bounded for both regimes using Lemmas \ref{lem:FTRL-L1}, \ref{lem:FTRL-L2}, and \ref{lem:FTRL-L3}. 
\end{proof}

\begin{fact}\label{fact:F1}
    $f^{\prime\prime}_t(x):\Rp \to \Rp$ are monotonically decreasing positive functions and $f^{*\prime}_t: \R\to\Rp$ are convex and monotonically increasing
\end{fact}
\begin{proof}
By definition $f^{\prime\prime}_t(x) = \tfrac12 \alpha_t^{-1} x^{-3/2} + \beta_t^{-1} x^{-1} > 0$, proving the first statement. 

Since $f_t$ are Legendre functions, we have $f^{*\prime\prime}_t(x^*) = (f^{\prime\prime}_t(f^{*\prime}_t(x^*)))^{-1} > 0$, showing that functions $f^{*\prime}_t$ are monotonically increasing. As both $f^{\prime\prime}_t(x)^{-1}$ and $f^{*\prime}_t(x^*)$ are increasing, their composition is also increasing, so $f^{*\prime\prime\prime}_t > 0$, showing that $f^{*\prime}_t$ are convex.
\end{proof}

\begin{fact}\label{fact:F2}
For every convex $F$, for $L \in \R^K$ and $c \in \R$:
\begin{equation*}
\cconjF(L + c \bone_K) = \cconjF(L) + c.
\end{equation*}
\end{fact}
\begin{proof}
By definition $\cconjF(L+c\bone_K) = \sup_{x\in\Delta([K])} \pair{x, L+c\bone_K} - F(x) = \sup_{x\in\Delta([K])}\pair{x, L} - F(x) + c$.
\end{proof}

\begin{fact}\label{fact:F3}
For every $x_t$, there exists $c \in \R$ such that:
\begin{equation*}
x_t = \nabla\cconjF_t(-\htLobs_t) = \nabla\conjF_t(-\htLobs_t + c \bone_K) = \nabla\conjF_t(\nabla F_t(x_t)).
\end{equation*}
\end{fact}
\begin{proof}
By the KKT conditions, there exists $c \in \R$ such that $x_t = \arg\sup_{x \in \Delta([K])} \pair{x, -\htLobs_t} - F_t(x)$ satisfies $\nabla F_t(x_t) = -\htLobs_t + c \bone_K$.
The rest follows by the standard property $\nabla F = (\nabla F^*)^{-1}$ for Legendre $F$.
\end{proof}

\begin{fact}\label{fact:F4}
For every Legendre function $F$ and $L \in \R^K$, it holds that $$\cconjF(L) \le \conjF(L),$$ with equality if and only if there exists $x \in \Delta([K])$ such that $L = \nabla F(x)$.
\end{fact}
\begin{proof}
The first statement follows from the definitions of convex conjugates as $\Delta([K]) \subset \R^K$. For the second statement, equality $\cconjF(L) = \conjF(L)$ would be equivalent to $$ \nabla F^*(L) = {\arg\sup}_{x\in\R^K} \pair{x, L} - F(x) = {\arg\sup}_{x\in\Delta([K])} \pair{x, L} - F(x) \in \Delta([K]),$$
and equivalently $L = \nabla F(x)$ for $x = \nabla F^*(L) \in \Delta([K])$
\end{proof}

\begin{fact}\label{fact:F5}
For every $x \in \Delta([K])$, $L \ge 0$, and $i \in [K]$ it holds that:
\begin{equation*}
\nabla\cconjF_t(\nabla F_t(x) - L)_i \ge \nabla\conjF_t(\nabla F_t(x) - L)_i.  
\end{equation*}

\end{fact}

\begin{proof}
Using a similar argument as in the proof of Fact~\ref{fact:F3}, by the KKT conditions, we can find $c \in \R$ such that $\nabla\cconjF_t(\nabla F_t(x) - L) = \nabla\conjF_t(\nabla F_t(x) - L + c\bone_K)$. By Fact~\ref{fact:F1} $f^{*\prime}_t$ is monotonically increasing, so the statement is equivalent to $c \ge 0$. It cannot be that $c < 0$, because otherwise it would hold that
\begin{align*}
    1 
    &= \sum_{i=1}^K (\nabla\cconjF_t(\nabla F_t(x) - L))_i\\
    &= \sum_{i=1}^K (\nabla \conjF_t(\nabla F_t(x) - L + c\bone_K))_i\\
    &= \sum_{i=1}^K f^{*\prime}_t(f'_t(x_i) - L_i + c)\\
    &< \sum_{i=1}^K f^{*\prime}_t(f'_t(x_i))\\
    &= 1.
\end{align*}

\end{proof}

\begin{fact}\label{fact:F6}
Let $D_F(x, y) = F(x) - F(y) - \pair{x - y, \nabla F(y)}$ denote the Bregman divergence of a function $F$. For every Legendre function $f$ with a monotonically decreasing second derivative, $x \in \text{dom}(f)$, and $l \ge 0$, such that $f'(x) - l \in \text{dom}(f^*)$, it holds that
\begin{equation*}
D_{f^*}(f'(x) - l, f'(x)) \le \frac{l^2}{2f''(x)}.
\end{equation*}
\end{fact}
\begin{proof}
By Taylor's theorem, there exists $\tilde{x} \in [f^{*\prime}(f'(x) - l), x]$ such that $$D_{f^*}(f'(x) - \ell, f'(x)) = \frac{\ell^2}{2f''(\tilde{x})}.$$ Since $\tilde{x} \le x$ and $f''$ is decreasing, we have $f''(\tilde{x})^{-1} \leq f''(x)^{-1}$, completing the proof.
\end{proof}

\begin{lemma}\label{lem:FTRL-L1}
    For every $t\in [T]$, the following holds true in the bandit setting:
    \begin{equation*}
        \E\sbr{\cconjF_t(-\htLobs_t - \htl_t) - \cconjF_t(-\htLobs_t) + \pair{x_t, \htl_t}} \le \sqrt{K} \alpha_t B_t^2.    
    \end{equation*}
    And in the full-information setting:
    \begin{equation*}
        \E\sbr{\cconjF_t(-\htLobs_t - \htl_t) - \cconjF_t(-\htLobs_t) + \pair{x_t, \htl_t}} \le \frac12 \beta_t B_t^2.    
    \end{equation*}
\end{lemma}
\begin{proof}
    Our proof builds on the argument from \cite{seldin20} and extends it to handle general estimators. We then apply this general result to both the bandit and full-information settings, incorporating modifications to address losses of time-varying scales in both cases. For both settings, we write
    \begin{align}
        \cconjF_t(-\htLobs_t - \htl_t) - \cconjF_t(-\htLobs_t) + \pair{x_t, \htl_t}
        &\myeq{a} \cconjF_t(\nabla F_t(x_t) - \htl_t) - \cconjF_t(\nabla F_t(x_t)) + \pair{x_t, \htl_t} \nonumber\\
        &\myle{b} \conjF_t(\nabla F_t(x_t) - \htl_t) - \conjF_t(\nabla F_t(x_t)) + \pair{x_t, \htl_t} \nonumber\\
        &= \sum_{i=1}^K D_{f_t^*}\rbr{f'_t(x_{t,i}) - \htl_{t,i}, f'_t(x_{t,i})},\label{eq:any-est-bound-L1}
    \end{align}
    where (a) follows from Facts~\ref{fact:F3} and then \ref{fact:F2}, while (b) follows from both parts of Fact~\ref{fact:F4}.
    In the bandit setting, with estimators $\htl_t = l_{t,A_t}x_{t,A_t}^{-1}\be_{A_t}$, \eqref{eq:any-est-bound-L1} can be further bounded as
    \begin{align*}
        \sum_{i=1}^K D_{f_t^*}\rbr{f'_t(x_{t,i}) - \htl_{t,i}, f'_t(x_{t,i})}
        &\le D_{f_t^*}\rbr{f'_t(x_{t,A_t}) - \htl_{t,A_t}, f'_t(x_{t,A_t})}\\
        &\myle{c} \frac12 (\htl_{t,A_t})^2 f_{t,1}''(x_{t,A_t})^{-1}\\
        &= \frac12 \rbr{l_{t,A_t} x_{t,A_t}^{-1}}^2 \rbr{2\alpha_t (x_{t,A_t})^{3/2}}\\
        &\le x_{t,A_t}^{-1/2} \alpha_t B_t^2,
    \end{align*}
    where (c) follows from Fact~\ref{fact:F6} and then bounding $f^{\prime\prime}_t$ with $f^{\prime\prime}_{t,1}$ from below.
    By taking expectation over each $A_t \sim x_t$, we obtain the bandit result:
    \begin{align*}
        \E\sbr{\cconjF_t(-\htLobs_t - \htl_t) - \cconjF_t(-\htLobs_t) + \pair{x_t, \htl_t}} \le  \E\sbr{\sum_{i=1}^K (x_{t,i})^{1/2} \cdot \alpha_t B_t^2} \le \sqrt{K} \alpha_t B_t^2.
    \end{align*}
    In the full-information setting, with estimators $\htl_t = l_t$, \eqref{eq:any-est-bound-L1}  can be bounded as
    \begin{align*}
        \sum_{i=1}^K D_{f_t^*}\rbr{f'_t(x_{t,i}) - \htl_{t,i}, f'_t(x_{t,i})}
        &\myle{d} \sum_{i=1}^K \frac12 (\htl_t)^{2} f_{t,2}''(x_{t,i})^{-1}\\
        &= \frac12 l_t^2 \sum_{i=1}^K \beta_t x_{t,i}\\
        &\le \frac12 \beta_t B_t^2,
    \end{align*}
    where (d) follows from Fact~\ref{fact:F6} and bounding $f^{\prime\prime}_t$ with $f^{\prime\prime}_{t,2}$ from below. This concludes the proof of the full-information result.
\end{proof}

\begin{lemma}\label{lem:FTRL-L2}
    For non-increasing learning rates $\alpha_t, \beta_t$, in both bandit and full-information settings, it holds almost surely that
    \begin{equation*}
        \sum_{t=1}^T \rbr{\cconjF_t (-\htL_t) - \cconjF_t (-\htL_{t+1})} - \htL_{T+1,i^*} \le 2\sqrt{K} \alpha_T^{-1} + \log(K) \beta_T^{-1}.
    \end{equation*}
\end{lemma}
\begin{proof}
    This proof repeats the argument from \cite{seldin20} without any notable changes.

    \noindent Let $\barx_t = {\arg\sup}_{x\in \Delta([K])} \pair{x, -\htL_t} - F_t(x)$, so that
    \begin{equation*}
        \cconjF_t(-\htL_t) =  \pair{\barx_t, -\htL_t} - F_t(\barx_t) = \sup_{x\in\Delta([K])} \pair{x, -\htL_t} - F_t(x).
    \end{equation*}
    By the definition of constrained convex conjugate, it holds that
    \begin{align*}
        &\cconjF_{T}(-\htL_{T+1}) \ge \pair{e_{i^*}, -\htL_{T+1}} - F_T(e_{i^*}) \ge -\htL_{T+1, i^*},\\
        &\cconjF_{t-1}(-\htL_t) \ge \pair{\barx_t, -\htL_t} - F_{t-1}(\barx_t).
    \end{align*}
    Plugging these inequalities into the LHS gives us
    \begin{align*}
        \sum_{t=1}^T \rbr{\cconjF_t (-\htL_t) - \cconjF_t (-\htL_{t+1})} &- \htL_{T+1,i^*}\\
        &\le \cconjF_1(-\htL_1) + \sum_{t=2}^T\rbr{-\cconjF_{t-1}(-\htL_t) + \cconjF_t(-\htL_t)}\\
        &\le -F_1(\barx_1) + \sum_{t=2}^T (F_{t-1}(\barx_t) - F_t(\barx_t))\\
        &\le \sup_{x\in\Delta([K])} -F_1(x) + \sum_{t=2}^T \sup_{x\in\Delta([K])} (F_{t-1}(x)-F_t(x))\\
        &= \sup_{x\in\Delta([K])} \rbr{(-\alpha_1^{-1})F_{0,1} + (-\beta_{1}^{-1}) F_{0,2}}(x)\\ 
        &+ \sum_{t=2}^T \sup_{x\in\Delta([K])}\rbr{(\alpha_{t-1}^{-1} - \alpha_t^{-1})F_{0,1} + (\beta_{t-1}^{-1} - \beta_{t}^{-1})F_{0,2}}(x)\\
        &\myeq{a} -F_1(\bone_K / K) + \sum_{t=2}^T (F_{t-1}(\bone_K / K) - F_{t}(\bone_K / K))\\
        &= -F_T(\bone_K / K)\\
        &=  2\sqrt{K} \alpha_T^{-1} + \log(K) \beta_T^{-1},
    \end{align*}
    where (a) follows from the fact that both learning rates are non-increasing and that $\bone_K/K$ minimizes both $F_{0,1}$ and $F_{0,2}$ on $\Delta([K])$.
\end{proof}

\begin{lemma}\label{lem:FTRL-L3}
    For every $t \in [T]$, in both bandit and full-information settings, it holds that
    \begin{equation*}
        \E\sbr{\cconjF_t(-\htLobs_t) - \cconjF_t(-\htLobs_t - \htl_t) - \cconjF_t (-\htL_t) + \cconjF_t (-\htL_{t+1})} \le \beta_t B_t \sum_{s \in \cW_t} B_s.
    \end{equation*}
\end{lemma}
\begin{proof}
    Similar to the proof of Lemma~\ref{lem:FTRL-L1}, here our proof extends the argument from \cite{seldin20} to analyze general estimators. Building on this generalized framework, we incorporate modifications to address losses of time-varying scales and apply this result to both bandit and full-information settings.

    Let $\htLmiss_t = \htL_t - \htLobs_t = \sum_{s\in \cW_t} \htl_s$ denote the sum of estimators whose values were determined but not observed by the start of round $t$. Consider function $\barx(z) = \nabla\cconjF_t(-\htLobs_t - z\htl_t)$. Then, for both bandits and full-information regimes, we write
    \begin{align}
        \cconjF_t(-\htLobs_t) - \cconjF_t(-\htLobs_t &- \htl_t) - \cconjF_t (-\htL_t) + \cconjF_t (-\htL_{t+1}) \nonumber\\
        &\myeq{a} 
        \int_{0}^{1} \pair{\htl_t, \nabla\cconjF_t(-\htLobs_t - z\htl_t)}dz 
        - \int_{0}^{1} \pair{\htl_t, \nabla\cconjF_t(-\htLobs_t - \htLmiss_t - z\htl_t)}dz \nonumber\\
        &\myeq{b} 
        \int_{0}^{1} \pair{\htl_t, \barx(z) - \nabla\cconjF_t(\nabla F_t(\barx(z)) - \htLmiss_t)}dz \nonumber\\
        &\myle{c}
        \int_{0}^{1} \pair{\htl_t, \barx(z) - \nabla\conjF_t(\nabla F_t(\barx(z)) - \htLmiss_t)}dz \nonumber\\
        &= \sum_{i=1}^K \int_{0}^{1} \htl_{t,i} (\barx_{i}(z) - f^{*\prime}_t (f'_t(\barx_{i}(z) - \htLmiss_{t,i}))) dz \nonumber\\
        &\myle{d}
        \sum_{i=1}^K \int_{0}^{1} \htl_{t,i} (f^{*\prime\prime}_t(f'_t(\barx_i(z)))) \htLmiss_{t,i} dz \nonumber\\
        &= \sum_{i=1}^K \int_{0}^{1} \htl_{t,i} ((f_t^{\prime\prime} \circ f^{*\prime}_t)(f'_t(\barx_i(z))))^{-1} \htLmiss_{t, i} dz \nonumber\\
        &= \sum_{i=1}^K \int_{0}^{1} \htl_{t,i} f_t^{\prime\prime}(\barx_i(z))^{-1} \htLmiss_{t,i} dz, \label{eq:any-est-bound-L3}
    \end{align}
    where (a) follows from the fundamental theorem of calculus, (b) substitutes $\barx(z)$ and applies Fact~\ref{fact:F3}, (c) applies Fact~\ref{fact:F5}, and (d) follows from the convexity of $f^{*\prime}_t$ by Fact~\ref{fact:F1}.
    In the bandit setting, with estimators $\htl_t = l_{t,A_t}x_{t,A_t}^{-1}\be_{A_t}$, \eqref{eq:any-est-bound-L3} can be further bounded as
    \begin{align*}
        \sum_{i=1}^K \int_{0}^{1} \htl_{t,i} f_t^{\prime\prime}(\barx(z))^{-1} \htLmiss_{t,i} dz
        &= \int_{0}^{1} \htl_{t,A_t} f_t^{\prime\prime}(\barx_{A_t}(z))^{-1} \htLmiss_{t, A_t} dz\\
        &\myle{e}
        \int_{0}^{1} \htl_{t,A_t} f_t^{\prime\prime}(x_{t,A_t})^{-1} \htLmiss_{t, A_t} dz\\
        &\le \int_{0}^{1} \htl_{t,A_t} f_{t,2}^{\prime\prime}(x_{t,A_t})^{-1} \htLmiss_{t, A_t} dz\\
        &= \int_{0}^{1} (l_{t,A_t}x_{t,A_t}^{-1})(\beta_t x_{t,A_t})\htLmiss_{t, A_t} dz\\
        &\le \beta_t B_t \htLmiss_{t,A_t},
    \end{align*}
    where (e) follows because $f^{\prime\prime}_t(x)$ is monotonically increasing by Fact~\ref{fact:F1} and for every $z\ge 0$ it holds that 
    \begin{equation}\label{eq:change-only-one}
        \barx_{A_t}(z) = (\nabla\cconjF_t(-\htLobs_t - zl_{t,A_t}x_{t,A_t}^{-1}\be_{A_t}))_{A_t} \le (\nabla\cconjF_t(-\htLobs_t))_{A_t} =  x_{t,A_t}.
    \end{equation}
    Equation~\eqref{eq:change-only-one} holds because $\nabla\cconjF_t(-L)_{A_t}$ decreases when the loss increases only in coordinate $A_t$.
    As $A_t$ and $\calbr{A_s: s\in \cW_t}$ are independent given $\calbr{A_s: s\in O_t}$, we have in expectation
    \begin{align*}
        \E\sbr{\htLmiss_{t,A_t}} &= \sum_{s\in \cW_t} \E\sbr{\htl_{s, A_t}} = \sum_{s\in \cW_t}\sum_{i=1}^K \E\sbr{\htl_{s, i} \ind(A_t = i)} =  \sum_{s\in \cW_t}\sum_{i=1}^K \E\sbr{\htl_{s, i}} \Pr{\calbr{A_t = i}} \le \sum_{s\in \cW_t} B_s,
    \end{align*}
    which concludes the proof for the bandit result.
    In the full-information setting, with estimators $\htl_t = l_t$, \eqref{eq:any-est-bound-L3} can be bounded as
    \begin{align*}
        \sum_{i=1}^K \int_{0}^{1} \htl_{t,i} f_t^{\prime\prime}(\barx_i(z))^{-1} \htLmiss_{t,i} dz
        &\le \sum_{i=1}^K \int_{0}^{1} \htl_{t,i} f_{t,2}^{\prime\prime}(\barx_i(z))^{-1} \htLmiss_{t,i} dz\\
        &= \sum_{i=1}^K \int_{0}^{1} l_{t,i} (\beta_t \barx_i(z)) \htLmiss_{t,i} dz\\
        &\le \beta_t B_t \sum_{i=1}^K \int_{0}^{1} \barx_i(z) \htLmiss_{t,i} dz\\
        &\le \beta_t B_t \sum_{s\in \cW_t} \int_{0}^{1} \sum_{i=1}^K \barx_i(z) l_{s,i} dz\\
        &\myle{f} \beta_t B_t \sum_{s\in\cW_t} B_s,
    \end{align*}
    where (f) follows from the fact that $\barx(z) \in \Delta([K])$ and each $l_{s,i} \le B_s$. This concludes the proof for the full-information result.
\end{proof}

\section{Batch Partitioning Algorithm for Delay Scheduling: Proof}\label{app:batching-proof}

In this section, we prove Theorem~\ref{thm:batch-bandit-fullinfo}. To do so, we first establish the more general Theorem~\ref{thm:batch-partition}, which applies to a broader class of learning rates. Theorem~\ref{thm:batch-bandit-fullinfo} then follows as a corollary.

We introduce a bit more batch-level. Extend the final batch with zero losses $l_t = 0$ and delays $d_t = 0$ for $t \in \sbr{T'b} \setminus \sbr{T}$. Then, for each batch $\tau \in [T']$, let $\cB_{\tau} = \calbr{(\tau-1)b +1, \ldots,\tau b}$, $L_{\tau}^b = \sum_{t\in \cB_{\tau}} l_t$, and $l_{\tau}^b = l_{u_{\tau}}$. Also, let $\htl^b_{\tau} = \htl_{u_{\tau}}$ as generated by algorithm. So, $\htLobs_{\tau} = \sum_{s:s+d^b_s<\tau} \htl^b_{s}$.

Note that sequence of independently sampled representatives $(u_{\tau})_{\tau=1}^{T'}$ determines scheduling behavior of Algorithm~\ref{alg:batch-partition}. Consider filtration $\cH_{\tau} = \sigma(u_1, ..., u_{\tau-1})$ for $\tau \in [T']$. 

\begin{fact}\label{fact:batching-assumption}
    Suppose $C\ge 2$ and $b\ge \frac{\dmax}{C-1}$. Then, Algorithm~\ref{alg:batch-partition} never exceeds maximum capacity $C$.
\end{fact}
\begin{proof}
    This trivially holds because, at any round, the tracking set can contain representative rounds from at most $\ceil{\dmax/b}$ previous batches, so the size of the tracking set at any point in time is at most $1 + \ceil{\dmax / b} \le 1 + (C-1) = C$. 
\end{proof}

\begin{theorem}\label{thm:batch-partition}
    Suppose that $b \ge \frac{\dmax}{C-1}$ and learning rates $\alpha_{\tau}$ and $\beta_{\tau}$ are $\cH_{\tau}$-measurable. Then, for the bandit regime, Algorithm \ref{alg:batch-partition} ensures that the expected regret satisfies:
    \begin{align*}
       \frac{\eR_{T}}{b} \le \E\sbr{\tsum_{\tau=1}^{T'} \rbr{\sqrt{K}\alpha_{\tau} + \beta_{\tau} \sigma^b_{\tau}} + 2\sqrt{K}\alpha_{T'}^{-1} + \log(K)\beta_{T'}^{-1}}.
    \end{align*}
    And for the full-information regime ($\alpha_{\tau} = \infty$):
    \begin{align*}
        \frac{\eR_{T}}{b} \le \E\sbr{\tsum_{\tau=1}^{T'} \beta_{\tau} (\sigma^b_{\tau}+1) + \log(K)\beta_{T'}^{-1}}.
    \end{align*}
\end{theorem}

\begin{proof}
    First of all, Fact~\ref{fact:batching-assumption} confirms that for this batch size $b$, capacity $C$ is never exceeded.  
    
    Next, using the fact that learning rates $\alpha_{\tau}, \beta_{\tau}$ are $\cH_{\tau}$-measurable, we will construct a randomization sequence $(X_{\tau})_{\tau=1}^{T'}$ of i.i.d. random variables independent of $\cH_{T'+1}$ such that each $A_{\tau}^b$ will be measurable with respect to $\sigma(\cF_{\tau}; X_{\tau})$, where $\cF_{\tau} = \sigma\rbr{\cH_{\tau}; X_1, \ldots, X_{\tau-1}}$ encapsulates all the randomness of the algorithm up to the start of batch $\tau$.
    \begin{itemize}[noitemsep,before=\vspace{-0.2cm}]
        \item For the base case $\tau =1$, $\cH_1 = \cF_1 = \calbr{\emptyset, \Omega}$, so $\alpha_1, \beta_1$, $\htLobs_1$, and $x_1$ are constants. Thus, we can just take $X_1 \sim \text{Unif}(0,1)$ independent of $\cH_{T'+1}$ that would determine $A_1^b$, i.e. $A_1^b$ would be $\sigma(\cF_{1}; X_1)$ measurable. 
        \item For the induction step, suppose that we constructed the first $\tau \in [T'-1]$ random variables $(X_{\tau})_{s=1}^{\tau}$. As $\alpha_{\tau+1}$, $\beta_{\tau+1}$, and $\htLobs_{\tau+1}$ are $\cF_{\tau+1}$-measurable, $x_{\tau+1}$ is as well. Thus, we can just take $X_{\tau+1} \sim \text{Unif}(0,1)$ independently of both $\cH_{T'+1}$ and $(X_s)_{s=1}^{\tau}$, so that it would determine randomness of $A_{\tau+1}^b$ based on $\cF_{\tau+1}$-measurable $x_{\tau+1}$. Then, $A_{\tau+1}$ would be indeed $\sigma(\cF_{\tau+1}; X_{\tau+1})$-measurable.    
    \end{itemize}
    Therefore, we have conditional independence $A_{\tau}^b \perp u_{\tau} \mid \cF_{\tau}$ for all $\tau\in [T']$ since $A_{\tau}^b$ is $\sigma(\cF_{\tau}; X_{\tau})$-measurable and $u_{\tau}$ is independent of random variables $\{u_s\}_{s=1}^{\tau-1} \cup \{X_s\}_{s=1}^{\tau}$.
    
    For action $a\in [K]$, let $R_{T', a}^b = \sum_{\tau=1}^{T'}(l^b_{\tau, A^b_{\tau}} - l^b_{\tau, a})$. Then, we can write
    \begin{align*}
        \eR_{T}
        &= \E\sbr{\tsum_{\tau = 1}^{T'} \sum_{t\in \cB_\tau} (l_{t, A_t} - l_{t, i^*})}\\
        &= \tsum_{\tau = 1}^{T'} \E\sbr{\E[\pair{\be_{A_{\tau}^b} - \be_{i^*}, L_{\tau}^b} \mid \cF_{\tau}]}\\
        &= \tsum_{\tau = 1}^{T'} \E\sbr{\pair{ \E[\be_{A_{\tau}^b} - \be_{i^*} \mid \cF_t] ,  L_{\tau}^b } } \\
        &\myeq{a} \tsum_{\tau = 1}^{T'} \E\sbr{\pair{ \E[\be_{A_{\tau}^b} - \be_{i^*} \mid \cF_t] , \E[l^b_{\tau}|\cF_t] b } }\\
        &\myeq{b} \E\sbr{\tsum_{\tau = 1}^{T'} (l^b_{\tau, A_{\tau}^b} - l^b_{\tau, i^*})b}\\
        &= \E\sbr{\E[\left. R^b_{T', i^*}\right| \cH_{T'+1}]} b,
    \end{align*}
    where (a) uses the fact that $L_{\tau}^b =  \E[l_{\tau}^b b] = \E[l_{\tau}^b b | \cF_t]$ as $u_{\tau}$ is independent of $\cF_t$ and (b) applies conditional independence $A_{\tau}^b \perp u_{\tau} | \cF_t$.
    
    For the fixed choice of representatives $u_{\tau}$, our algorithm effectively runs Delayed FTRL from \cite{seldin20} over $T'$ rounds with oblivious losses $l^b_{\tau}$ and delays $d^b_{\tau}$. Hence, we can bound this expected regret conditioned on the choice of representatives $\E[R^b_{T', i^*}|\cH_{T'+1}]$ via Theorem~\ref{thm:FTRL-no-scale} for $T'$ rounds and loss scales $B_{\tau} = 1$.
    For the bandit regime, we have:
    \begin{align*}
        \E\sbr{\left. R^b_{T', i^*}\right| \cH_{T'+1}}
        &\le \tsum_{\tau=1}^{T'} \rbr{\sqrt{K}\alpha_{\tau} + \beta_\tau \sigma^b_{\tau}} + 2\sqrt{K}\alpha_{T'}^{-1} + \log(K)\beta_{T'}^{-1},
    \end{align*}
    and for the full-information regime:
    \begin{align*}
        \E\sbr{\left. R^b_{T', i^*}\right| \cH_{T'+1}}
        &\le \tsum_{\tau=1}^{T'} \beta_{\tau} (\sigma^b_{\tau} + 1) + \log(K)\beta_{T'}^{-1}.
    \end{align*}
    By taking expectation, we obtain the stated bound on $\eR_T$.
\end{proof}

\begin{lemma}\label{lem:batch-sigmas-delays}
    For every batch $\tau \in [T']$, number of outstanding batch delays $\sigma^b_{\tau}$ is $\cH_{\tau}$-measurable. It almost surely holds that $\sum_{\tau=1}^{T'} \sigma^b_{\tau} = \sum_{\tau=1}^{T'} d^b_{\tau}$. Plus, it holds that
    $\E\sbr{\sum_{\tau=1}^{T'} d^b_{\tau}} \le \frac{D}{b^2} + T'$.
\end{lemma}
\begin{proof}
    The first two statements are trivial. To prove the third one, we write
    \begin{align*}
        \sum_{\tau=1}^{T'} \E[d^b_{\tau}] = \sum_{\tau=1}^{T'} \E\sbr{\tceil{\tfrac{u_{\tau} + d_{u_\tau}}{b}} - \tceil{\tfrac{u_{\tau}}{b}}} \le \sum_{\tau=1}^{T'} \E\sbr{\tceil{\tfrac{d_{u_\tau}}{b}}} \le \sum_{\tau=1}^{T'} \E\sbr{d_{u_\tau}/b + 1} = \frac{D}{b^2} + T'. 
    \end{align*}
\end{proof}

\begin{lemma}[\cite{seldin14}, Lemma 8]\label{lem:sqrt-frac-ineq}
For a sequence $(x_t)_{t=1}^T$ on $[0,\infty)$, let $\eta_t = (\sum_{s=1}^t x_s)^{-0.5} \in (0,\infty]$. Then, with the convention that $x_t \eta_t = 0$ when $x_t = 0$, it holds that
$\tsum_{t=1}^T x_t \eta_t \le 2\eta_T^{-1}$.
\end{lemma}

\begin{proof}(Theorem~\ref{thm:batch-bandit-fullinfo})
    In both bandit and full-information regimes, the chosen learning rates are $\cH_{\tau}$-measurable because each $\sigma^b_{\tau}$ is $\cH_{\tau}$-measurable. Then, in the bandit regime, by Theorem~\ref{thm:batch-partition}, we have
    \begin{align*}
        \frac{\eR_T}{b}
        &\le \E\sbr{\tsum_{\tau=1}^{T'} \rbr{\sqrt{K}\alpha_{\tau} + \beta_\tau \sigma^b_{\tau}} + 2\sqrt{K}\alpha_{T'}^{-1} + \log(K)\beta_{T'}^{-1}}\\
        &\myle{a} 4\sqrt{T'K} + 3\E\sbr{\sqrt{\tsum_{\tau\in[T']} \sigma^b_{\tau}}}\sqrt{\log(K)}\\
        &\myle{b} 4\sqrt{T'K} + 3\sqrt{D/b^2 + T'}\sqrt{\log(K)}.
    \end{align*}
    where (a) applies Lemma~\ref{lem:sqrt-frac-ineq} for our choice of the learning rates, (b) applies Jensen's inequality and Lemma~\ref{lem:batch-sigmas-delays}. Consequently, the expected regret for the bandit regime satisfies:
    \begin{align*}
        \eR_{T} &\le \rbr{4\sqrt{T'K} + 3\sqrt{(D/b^2 + T')\log(K)}} b\\
        &= 4\sqrt{\tceil{T/b}b^2 K} + 3\sqrt{(D+\tceil{T/b}b^2)\log(K)}\\
        &\le 8\sqrt{TbK} + 3\sqrt{D\log(K)} + 6\sqrt{Tb\log(K)}\\
        &\le 14\sqrt{TbK} + 3\sqrt{D\log(K)}.
    \end{align*}
    Similarly, in the full-information regime, by Theorem~\ref{thm:batch-partition}, we have
    \begin{align*}
        \frac{\eR_T}{b}
        &\le \E\sbr{\tsum_{\tau=1}^{T'} \beta_{\tau} (\sigma^b_{\tau}+1) + \log(K)\beta_{T'}^{-1}}\\
        &\myle{c} 3\E\sbr{\sqrt{\tsum_{\tau\in[T']} (\sigma^b_{\tau}+1)}}\sqrt{\log(K)}\\
        &\myle{d} 3\sqrt{D/b^2 + 2T'}\sqrt{\log(K)}.
    \end{align*}
    where (c) applies Lemma~\ref{lem:sqrt-frac-ineq} for our choice of the learning rates (d) applies Jensen's inequality and Lemma~\ref{lem:batch-sigmas-delays}. Consequently, the expected regret for the full-information regime satisfies:
    \begin{align*}
        \eR_{T} &\le \rbr{3\sqrt{D/b^2 + 2T'}\sqrt{\log(K)}} b\\
        &= 3\sqrt{(D+2\tceil{T/b}b^2)\log(K)}\\
        &\le 12\sqrt{Tb\log(K)} + 3\sqrt{D\log(K)}.
    \end{align*}
\end{proof}

\subsection{Application to Fixed Delays}\label{app:fixed-delays}

In this subsection, we consider Delay Scheduling under the assumption of fixed delays, which are all equal to the same $0\le d \le T$. As the player knows $d$ before the start of each round\footnote{We assume that $d$ is known to the player at the start of the game. Otherwise, it can be inferred within the first $d$ rounds, during which no feedback is received.}, the most restrictive scheduling framework to consider in this setting is Clairvoyant Non-\Preemptive.

\begin{proof}(Theorem~\ref{thm:fixed-delays})
    As per the comment above, it will suffice to consider the clairvoyant preemptive framework for the lower bound and the clairvoyant non-preemptive one for the upper bound.  
    See Theorem~\ref{thm:fixed-delays-lb} for the lower bound. The upper bound follows from Theorem~\ref{thm:batch-bandit-fullinfo} by applying non-preemptive Algorithm~\ref{alg:batch-partition} with batch size $b = \max\{1, \tceil{\tfrac{d}{C-1}}\} \ge \tceil{\tfrac{\dmax}{C-1}}$. Then, for the bandit regime, we have $$\eR_T = O(\sqrt{TbK + D\log(K)}) = O(\sqrt{TK(1 + d/C) + Td\log(K)}),$$
    and for the full-information regime:
    $$\eR_T = O(\sqrt{(Tb + D)\log(K)}) = O(\sqrt{T(1 + d/C + d)\log(K)}) = O(\sqrt{T(d+1)\log(K)}).$$
\end{proof}

\begin{theorem}[Fixed delays lower bound]\label{thm:fixed-delays-lb} For $K \le \tfloor{\frac{CT}{d+1}}$ in the bandit regime, the minimax regret of Delay Scheduling with fixed delays is of the order
    \begin{align*}
       \Omega\rbr{\sqrt{TK (1+\tfrac{d}{C})} + \sqrt{Td\log(K)}}.
    \end{align*}
    And for arbitrary $K$ in the full-information regime, regret is of the order
    \begin{align*}
       \Omega\rbr{\sqrt{T(d+1)\log(K)}}.
    \end{align*}
\end{theorem}

\begin{proof}
    We begin with the full-information case. Since the regret of Delay Scheduling with fixed delays is lower bounded by that of Delayed Online Learning with the same delays, the minimax result of \cite{weinberger2002} for the full-information regime implies that
    \begin{align*}
        \eR_T = \Omega\rbr{\sqrt{T(d+1)\log(K)}}.
    \end{align*}
    To derive a lower bound on the regret in the bandit regime, we reduce from both Delayed Bandits with fixed delays and Label-Efficient Bandits. Notably, for feedback from round $t$ to be observed, it must satisfy $t \in S_{\tau}^1$ for all $\tau \in \{t, t+1, \ldots, t+d\}$. Consequently, with probability one, we have
    \begin{align*}
        \tsum_{t=1}^T Z_t (d+1) \le \sum_{t=1}^T \sum_{\tau = t}^{\min\{t+d, T\}} \ind(t\in S_{\tau}^1) = \sum_{t=1}^T |S_t^1| \le CT. 
    \end{align*}
    Therefore, the player can observe losses from no more than $M = \tfloor{\frac{CT}{d+1}}$ different rounds. Note that $K \le M$ by assumption. Thus, from Theorem~\ref{thm:audibert10-lb} and \ref{thm:cesa16-lb} for the bandit regime, we have
    \begin{align*}
        \eR_T 
        = \Omega\rbr{\max\calbr{\sqrt{\tfrac{T^2K}{M}}, \sqrt{TK} + \sqrt{Td\log(K)}}}
        = \Omega\rbr{\sqrt{TK(1+\tfrac{d}{C})} + \sqrt{Td\log(K)}},
    \end{align*}
    via the reductions from Label Efficient Bandits and Delayed Bandits.
\end{proof}

\noindent \textit{Remark:} The regret in Theorem~\ref{thm:fixed-delays-lb} is already linear for $K = \lfloor \frac{CT}{d+1} \rfloor$. This shows that considering $K \ge \frac{CT}{d+1}$ is unnecessary, as regret remains linear.

\section{General Scheme for Scheduling and Learning}\label{app:ftrl-scheduler}

In this section, we prove Theorem~\ref{thm:ftrl-scheduler} and, as a corollary, present Theorem~\ref{thm:ftrl-scheduler-no-sigmas}, which replaces dependence on $\cW_t$ with $d_t$.

\begin{proof}(Theorem~\ref{thm:ftrl-scheduler})
    To show that $\cA$ is an \obsind delay scheduling algorithm, we will construct a randomization sequence $(X_t^{\cA})_{t=1}^T$ of i.i.d. random variables
    satisfying \Cref{defn:nice}.
    \begin{itemize}[noitemsep,before=\vspace{-0.2cm}]
        \item For the base case $t=1$, $\cF_1 = \calbr{\emptyset, \Omega}$, so $\alpha_1, \beta_1$, $\htLobs_1$, and $x_1$ would be constants. Thus, we can just take $X_1^{\cA} \sim \text{Unif}(0,1)$ independent of the sequence $(X_t^{\cS})_{t=1}^T$ that would determine $A_1$, i.e. $A_1$ would be $\cF_t^{\cA} = \sigma(\cF_1; X_1^{\cA})$ measurable. 
        \item For the induction step, suppose that we constructed the first $t \in [T-1]$ random variables $(X_s^{\cA})_{s=1}^t$. As $\alpha_{t+1}$, $\beta_{t+1}$, and $\htLobs_{t+1}$ are $\cF_{t+1}$-measurable, $x_{t+1}$ is as well. Thus, we can just take $X_{t+1}^{\cA} \sim \text{Unif}(0,1)$ independently of both $(X_s^{\cS})_{s=1}^T$ and $(X_s^{\cA})_{s=1}^{t}$, so that it would determine randomness of $A_{t+1}$ based on $\cF_{t+1}$-measurable $x_{t+1}$. Then, $A_{t+1}$ would be indeed $\cF_{t+1}^{\cA} = \sigma(\cF_{t+1}; X_{t+1}^{\cA})$ -measurable.    
    \end{itemize}
    Thus, algorithm $\cA$ can indeed be formalized as an \obsind delay scheduling algorithm. Moreover, we have conditional independence $A_t \perp Z_t \mid \cF_t$ for all $t\in [T]$ since $A_t$ is $\sigma(\cF_t; X_{t}^{\cA})$-measurable and $Z_t$ is $\sigma(\cF_t; {X_{t}^{\cS}})$-measurable, with $X_{t}^{\cA}$ and $X_{t}^{\cS}$ being independent random variables.
    
    Let $z_t = Z_t/p_t$ and $e_t = \ind(|S_t^0| < C)$, so that $\E[z_t \mid \cF_t] = e_t$. Note that rounds where capacity is exceeded can be forfeited for a price of $1$ per round in the regret bound, as follows:
    \begin{align}
        \eR_T 
        &= \E\sbr{\tsum_{t=1}^T (l_{t,A_t} - l_{t,i^*})}\notag\\
        &\le  \E\sbr{\tsum_{t=1}^T (e_t l_{t,A_t} - e_t l_{t,i^*})} + \E\sbr{\tsum_{t=1}^T (1-e_t)}.\label{eq:thm-scheduler}
    \end{align}
    For the second term in \eqref{eq:thm-scheduler}, we have
    \begin{equation*}
        \E\sbr{\tsum_{t=1}^T (1-e_t)} = \E\sbr{\tsum_{t=1}^T \ind(|S_t^0|=C)} = \tsum_{t=1}^T \Pr(|S_t^0|=C).
    \end{equation*}
    To analyze the first term in \eqref{eq:thm-scheduler}, let $\wtl_t = z_t l_t$ and $\wtL_t = \sum_{s=1}^{t-1} \wtl_s$. Then, write
    \begin{align*}
        \E[\wtl_t]
        = \E\sbr{\E[z_t l_t \mid \cF_t]}
        = \E\sbr{\E[z_t \mid \cF_t] l_t}
        = \E[e_t l_t].
    \end{align*}
    Moreover, using the fact that $Z_t$ and $A_t$ are independent when conditioned on $\cF_t$, we have
    \begin{align*}
        \E[\wtl_{t, A_t}]
        &= \E\sbr{\E[\pair{\be_{A_t}, l_t} z_t \mid \cF_t]}\\
        &\myeq{a} \E\sbr{\E[\pair{\be_{A_t}, l_t} \mid \cF_t]\, \E[z_t | \cF_t]}\\
        &= \E\sbr{\E[\pair{\be_{A_t}, l_t} \mid \cF_t]\, e_t}\\
        &\myeq{b} \E\sbr{e_t \pair{\be_{A_t}, l_t}}\\
        &=\E[e_t l_{t, A_t}],
    \end{align*}
    where (a) follows from the fact $A_t\perp Z_t \mid \cF_t$ and (b) applies the defintion of conditional expectation for event $\{|S_t^0| < C\} \in \cF_t^{\cS} \subseteq \cF_t$.
    This allows us to present the first term in \eqref{eq:thm-scheduler} as follows:
    \begin{align*}
        \E\sbr{\tsum_{t=1}^T (e_t l_{t,A_t} - e_t l_{t,i^*})}
        &= \E\sbr{\tsum_{t=1}^T (\wtl_{t,A_t} - \wtl_{t,i^*})}\\
        &= \E\sbr{\E\sbr{\left.\tsum_{t=1}^T \wtl_{t, A_t} - \wtL_{T+1,i^*}\right| \cF_{T+1}^{\cS}}}\\
        &=: \E\sbr{\widetilde\eR_{T, i^*}}.
    \end{align*}
    Here, $\widetilde\eR_{T, i^*}$ represents the expected regret against adversary $i^*$ for Delayed FTRL with time-varying loss scales $B_t = z_t$, losses $\wtl_t = l_t z_t \in [0,B_t]$, delays $d_t$, and learning rates $\alpha_t, \beta_t$. Importantly, even though the Delay Scheduling Algorithm~\ref{alg:FTRL-with-scheduler} does not have access to $z_t$ at round $t + d_t$ when no observation occurs ($Z_t = 0$), the reduction to the analysis of Delayed FTRL with time-varying loss scales still holds. This is because $\wtl_t = 0$ in such cases, and applying a zero loss to FTRL does not affect the algorithm's behavior. 
    
    Since $\wtl_t$, $\alpha_t$, and $\beta_t$ are all $\cF_{T+1}^{\cS}$-measurable, they act as constants when conditioned on $\cF_{T+1}^{\cS}$. Applying Theorem~\ref{thm:FTRL-no-scale}, we can bound the first term in \eqref{eq:thm-scheduler} as follows. For the bandit regime:
    \begin{align*}
        \E\sbr{\widetilde\eR_{T, i^*}} 
        \le \E\sbr{\tsum_{t=1}^T \rbr{\sqrt{K}\alpha_t z_t^2 + \beta_t z_t \sum_{s\in\cW_t}z_s} + 2\sqrt{K}\alpha_T^{-1} + \log(K)\beta_T^{-1}}
    \end{align*}
    and for the full-information regime:
    \begin{align*}
        \E\sbr{\widetilde\eR_{T, i^*}} 
        \le \E\sbr{\tsum_{t=1}^T \rbr{\beta_t z_t^2 + \beta_t z_t \sum_{s\in\cW_t}z_s} + \log(K)\beta_T^{-1}}.
    \end{align*}
    This concludes the proof of Theorem~\ref{thm:ftrl-scheduler}.
\end{proof}

\begin{theorem}\label{thm:ftrl-scheduler-no-sigmas}
    Under the same conditions as Theorem~\ref{thm:ftrl-scheduler}, the expected regret is also bounded in the bandit regime as:
    \begin{equation*}
        \eR_T 
        \le \E\sbr{\tsum_{t=1}^T \rbr{\sqrt{K}\alpha_t \frac{Z_t}{p_t^2} + \beta_t \frac{Z_t}{p_t} d_t} + 2\sqrt{K}\alpha_T^{-1} + \log(K)\beta_T^{-1}} + \tsum_{t=1}^T \Pr(|S_t^0| = C).
    \end{equation*}
    and in the full-information regime as:
    \begin{equation*}
        \eR_T 
        \le \E\sbr{\tsum_{t=1}^T \rbr{\beta_t \frac{Z_t}{p_t^2} + \beta_t \frac{Z_t}{p_t} d_t} + \log(K)\beta_T^{-1}} + \tsum_{t=1}^T \Pr(|S_t^0| = C).
    \end{equation*}
\end{theorem}
\begin{proof}
    Let $z_t = Z_t/p_t$ and $e_t = \ind(|S_t^0| < C)$, so that $\E[z_t \mid \cF_t] = e_t$. Based on the result of Theorem~\ref{thm:ftrl-scheduler}, for this theorem to hold, it will suffice to show $\E[\tsum_{t=1}^T (\beta_t z_t \tsum_{s\in \cW_t} z_s)] \le \E[\tsum_{t=1}^T \beta_t z_t d_t]$.
    Using the fact that $\beta_t$ is non-increasing and $\cW_t \subseteq [t-1]$, we can write
    \begin{align*}
        \E\sbr{\tsum_{t=1}^T \rbr{\beta_t z_t \tsum_{s\in \cW_t} z_s}} &\le \E\sbr{\tsum_{t=1}^T \rbr{z_t \tsum_{s\in \cW_t} \beta_s z_s}}\\
        &= \E\sbr{\tsum_{t=1}^T \rbr{\beta_t z_t \tsum_{s=t+1}^{t+d_t} z_s}}\\
        &= \tsum_{t=1}^T \tsum_{s=t+1}^{t+d_t} \E[\beta_t z_t z_s].
    \end{align*}
    Finally, for every $t, s \in [T]$ such that $s > t$, we have that
    \begin{align*}
         \E[\beta_t z_t z_s] 
         &= \E[\E[\beta_t z_t z_s \mid \cF_s^{\cS}]]\\
         &\myeq{a} \E[\beta_t z_t  \E[z_s \mid \cF_s^{\cS}]]\\
         &= \E[\beta_t z_t e_s]\\
         &\le \E[\beta_t z_t],
    \end{align*}
    where (a) follows from the fact that $\beta_t$ and $Z_t$ are $\cF_s^{\cS}$ measurable for $s > t$. In conclusion,
    \begin{equation*}
        \E\sbr{\tsum_{t=1}^T \rbr{\beta_t z_t \tsum_{s\in \cW_t} z_s}} \le \tsum_{t=1}^T \tsum_{s=t+1}^{t+d_t} \E[\beta_t z_t z_s] \le \E\sbr{\tsum_{t=1}^T \beta_t z_t d_t}.
    \end{equation*}
\end{proof}

\section{Analysis of Scheduling Policies}\label{app:proxy-delays}

\begin{fact}\label{fact:pareto-tail}
    For every $t\in \N$ and $d \in \Z_{\ge 0}$, $\Pr(\wtd_t \ge d) = \min\{1, \frac{C}{(1+\alpha)\nu_t}\cdot \frac{1}{d+1}\}$.
\end{fact}
\begin{proof}
    Sample $\bar d \sim \text{Pareto}\rbr{c, 1}$ for $c = \tfrac{C}{(1+\alpha)\nu_t}$ so that $\wtd_t$ has the same distribution as $\lfloor\bar d - 1\rfloor$. Since $\bar d$ has cumulative distribution function $F_{\bar d}(x) = \ind(x > c) (1 - \frac{c}{x})$ and $d \in \Z_{\ge 0}$, we can write
    \begin{align*}
        \Pr(\wtd_t \ge d) = \Pr(\lfloor\bar d - 1\rfloor \ge d) = \Pr(\bar d \ge d+1) = \min\calbr{1, \tfrac{c}{d+1}}.
    \end{align*}
\end{proof}

\begin{fact}\label{fact:proper-seq}
    For sequence $\nu_{t} = 2 H_t$, it holds that $\sum_{s=1}^{t} \frac{1}{\nu_s (t-s+1)} \le 1$ for every $t\in \N$. 
\end{fact}

\begin{proof}
    For every $t\in \N$, we can write
    \begin{align*}
        \tsum_{s=1}^{t} \tfrac{1}{\nu_s (t-s+1)} 
        &= \tsum_{s=1}^{\tceil{t/2}} \tfrac{1}{2H_s (t-s+1)} + \tsum_{s=\tceil{t/2}+1}^{t} \tfrac{1}{2H_s (t-s+1)}\\
        &\le \tsum_{s=1}^{\tceil{t/2}} \tfrac{1}{2 \tceil{t/2}} + \tsum_{s=\tceil{t/2}+1}^{t} \tfrac{1}{2H_{\tceil{t/2}} (t-s+1)}\\
        &\le \tfrac{\tceil{t/2}}{2\tceil{t/2}} + \tfrac{H_{\tceil{t/2}}}{2H_{\tceil{t/2}}}\\
        &= 1.
    \end{align*}
\end{proof}

\noindent Next, we prove Lemmas~\ref{lem:overflow-non-preemptive} and \ref{lem:overflow-bound} via the Chernoff multiplicative bound.

\begin{proof}(Lemma~\ref{lem:overflow-non-preemptive})
    Clairvoyance allows us to run Scheduler~\ref{alg:scheduler-bernoulli} with $p_t = \min\{1, \frac{C}{(1+\alpha)\nu_t}\cdot \frac{1}{d_t+1}\}$.\\ Suppose $b_t \sim \text{Ber}(p_t)$ is sampled every round and round $t$ is scheduled if and only if $b_t = 1$ and $|S_t^0| < C$, i.e. $Z_t = b_t \ind(|S_t^0| < C)$. Consider arbitrary round $t$. Due to the structure of this scheduler, the tracking set $S_t^0 = \{s\in[t-1]: s+d_s \ge t, Z_s = 1\}$. Thus, we can write
    \begin{align*}
        |S_t^0| = \sum_{s=1}^{t-1} Z_s \ind(s+d_s \ge t) \le \sum_{s=1}^{t-1} b_s \ind(d_s \ge t-s) =: \wtsig_t.
    \end{align*}
    The rest of the proof proceeds by bounding $\Pr(\wtsig_t \ge C)$ via the Chernoff bound.
    Note that $\wtsig_t$ can be written as a finite sum of independent Bernoulli random variables, because $b_s$ are sampled independently. The choice of $p_s$ and Fact \ref{fact:proper-seq} allows us to write 
    \begin{align*}
        \E[\wtsig_t] = \sum_{s=1}^{t-1} p_s \ind(d_s \ge t-s) = \sum_{s=1}^{t-1} \min\calbr{1, \tfrac{C}{(1+\alpha)\nu_s}\cdot \tfrac{1}{t-s+1}} \le \tfrac{C}{(1+\alpha)}.
    \end{align*}
    Let $\alpha' = \frac{C}{\E[\wtsig_t]} - 1 \ge \alpha$. Then, the Multiplicative Chernoff bound (e.g., see Theorem 2.3.1 from \cite{vershynin_book}) grants
    \begin{align*}
        \Pr(\wtsig_t \ge C) 
        &= \Pr(\wtsig_t \ge (1+\alpha') \E[\wtsig_t])\\
        &\le \rbr{\tfrac{e^{\alpha'}}{(1+\alpha')^{1+\alpha'}}}^{\frac{C}{1+\alpha'}}\\
        &= \exp\rbr{C\rbr{\tfrac{\alpha'}{1+\alpha'} - \ln(1+\alpha')}}\\
        &\myle{a} \exp\rbr{C\rbr{\tfrac{\alpha}{1+\alpha} - \ln(1+\alpha)}}\\
        &\myle{b} \delta,
    \end{align*}
    where (a) follows from the fact that function $f(x) = 1- \frac{1}{1+x}- \ln(1+x)$ is decreasing on the domain $(0, \infty)$ and $0 < \alpha \le \alpha'$ and (b) holds by \eqref{eq:chernoff-condition}. This concludes the proof.
\end{proof}

\begin{proof}(Lemma~\ref{lem:overflow-bound})
    The first part of the theorem follows from the fact that $\wtd_t$ is independent from $S_t^0$ and Fact~\ref{fact:pareto-tail}. Here, we write
    \begin{equation*}
        \E[Z_t \mid |S_t^0| < C] = \E[\ind(\wtd_t \ge d_t)\ind(|S_t^0|< C) \mid |S_t^0| < C] = \E[\ind(\wtd_t \ge d_t)] = \Pr(\wtd_t \ge d_t).
    \end{equation*}
    Consider arbitrary $t \in [T]$ and let $\wtsig_t = \sum_{s=1}^{t-1} \ind(s + \wtd_s \ge t)$ denote the number of outstanding proxy delays. Since the size of the tracking set $S$ cannot exceed the number of outstanding proxy delays at the start of round $t$, it is sufficient to verify that $\Pr(\wtsig_t \ge C) \le \delta$.
    
    Note that $\wtsig_t$ can be written as a finite sum of independent Bernoulli random variables $\wtsig_t = \sum_{s=1}^{t-1}\ind(\wtd_s \ge t-s)$, because proxy delays $\wtd_s$ are sampled independently from distributions $\gD_s$. From Facts~\ref{fact:pareto-tail} and \ref{fact:proper-seq}, we can write 
    \begin{align*}
        \E[\wtsig_t] = \sum_{s=1}^{t-1} \Pr(\wtd_s \ge t-s) = \sum_{s=1}^{t-1} \min\calbr{1, \tfrac{C}{(1+\alpha)\nu_s}\cdot \tfrac{1}{t-s+1}} \le \tfrac{C}{(1+\alpha)}.
    \end{align*}
    The rest of the proof follows via the Multiplicative Chernoff bound, exactly as was done in the proof of Lemma~\ref{lem:overflow-non-preemptive} above.
\end{proof}

\section{Upper Bounds for Clairvoyant or Preemptive Settings: Proofs}\label{app:NCP-CNP-results}

To prove \Cref{thm:CNP-results,thm:NCP-results}, note that both considered Schedulers~\ref{alg:scheduler-bernoulli} and \ref{alg:scheduler-proxy-delays} are \precom and quantified by the same sequence $p_t = \Pr(\wtd_t \ge d_t) = \min\{1, \frac{C}{(1+\alpha)\nu_t}\cdot \frac{1}{d_t+1}\} = \mu_t^{-1}$, where $\wtd_t$ is the proxy delay sampled from distribution $\gD_t$. Also, note that $p_t$ (and $\mu_t$) is computable whenever $d_t$ is known, so in clairvoyant frameworks that occurs for every round $t$ during round $t$, and in non-clairvoyant preemptive frameworks this occurs for every round $t$ with observed feedback during round $t+d_t$. Therefore, as $z_t \ne 0$ if and only if feedback from round $t$ arrived in round $t+d_t$ and $d_t$ was observed, all of these learning rates are computable using only available information.

Additionally, all the considered learning rates are $\cF_t^{\cS}$-measurable. For \Cref{thm:CNP-results} they are even constant, and for \Cref{thm:NCP-results} they are determined by the $\cF_t^{\cS} = \sigma(\wtd_1, \ldots, \wtd_{t-1})$. 

As we consider $z_t = Z_t/p_t$, note that $\E[z_t] = \E[Z_t\mu_t] \le 1$ and $\E[z_t^2] = \E[Z_t\mu_t^2] \le \mu_t$.

\vspace{10pt}
\begin{proof}(\Cref{thm:CNP-results})
    From Theorem~\ref{thm:ftrl-scheduler-no-sigmas} (modified Theorem~\ref{thm:ftrl-scheduler}) for bandits, it follows that
    \begin{align*}
        \eR_T 
        &\le \E\sbr{\tsum_{t=1}^T \rbr{\sqrt{K}\alpha_t z_t^2 + \beta_t z_t d_t} + 2\sqrt{K}\alpha_T^{-1} + \log(K)\beta_T^{-1}} + \tsum_{t=1}^T \Pr(|S_t^0| = C).\\
        &\myle{a} 
        \sqrt{K}\tsum_{t=1}^T \alpha_t \mu_t + \sum_{t=1}^T \beta_t d_t + 2\sqrt{K}\alpha_T^{-1} + \log(K)\beta_T^{-1} + \delta T \\
        &\myle{b} 
        4\sqrt{K}\sqrt{\tsum_{t=1}^T \mu_t} + 3\sqrt{\log(K)}\sqrt{\tsum_{t=1}^T d_t} + \delta T\\
        &\le 4\sqrt{K}\sqrt{T + \tfrac{(1+\alpha)\nu_t}{C}(D+T)} + 3\sqrt{D\log(K)} + \delta T,
    \end{align*}
    where (a) substitutes $\E[z_t] \le 1$ and $\E[z_t^2] \le \mu_t$ and applies Lemma~\ref{lem:overflow-non-preemptive}, (b) applies Lemma~\ref{lem:sqrt-frac-ineq} for our choice of the learning rates. Similarly, in the full-information regime, from Theorem~\ref{thm:ftrl-scheduler-no-sigmas}, we have
    \begin{align*}
        \eR_T 
        &\le \E\sbr{\tsum_{t=1}^T \rbr{\beta_t z_t^2 + \beta_t z_t d_t} + \log(K)\beta_T^{-1}} + \tsum_{t=1}^T \Pr(|S_t^0| = C)\\
        &\myle{c} \tsum_{t=1}^T \beta_t (\mu_t + d_t) + \log(K)\beta_T^{-1}  + \delta T \\
        &\myle{d} 3\sqrt{\log(K)}\sqrt{\tsum_{t=1}^T (\mu_t + d_t)} + \delta T\\
        &\le 3\sqrt{\log(K)} \sqrt{(D + T)(1+\tfrac{(1+\alpha)\nu_T}{C})} + \delta T,
    \end{align*}
    where again (c) substitutes $\E[z_t] \le 1$ and $\E[z_t^2] \le \mu_t$ and applies Lemma~\ref{lem:overflow-non-preemptive}, (d) applies Lemma~\ref{lem:sqrt-frac-ineq} for our choice of the learning rates. 
\end{proof}

\begin{proof}(\Cref{thm:NCP-results}) Note that for each round $t$, set $S_t^1$ contains at most $C$ rounds from the set $\cW_t \cup \{t\}$ and for each $s \in S_t^1 \subseteq [t]$, it holds that $z_s \le \mu_{\max,t}$ and $d_s \le \dmax$. Therefore, our choice of the learning rates uses only information available in non-clairvoyant, \preemptive scheduling to guarantee that in the bandit regime:
\begin{equation*}
     \alpha_t^{-1} \ge \sqrt{\tsum_{s\in [t]} z_s^2}, \quad\quad \beta_t^{-1} \ge  \log(K)^{-1/2}  \sqrt{\tsum_{s\in [t]} z_s d_s},
\end{equation*}
and in the full-information regime:
\begin{equation*}
    \beta_t^{-1} \ge  \log(K)^{-1/2} \sqrt{\tsum_{s\in [t]} z_s (z_s + d_s)}.
\end{equation*}
Then, in the bandit regime, from Theorem~\ref{thm:ftrl-scheduler-no-sigmas}, it follows that
\begin{align*}
    \eR_T 
    &\le \E\sbr{\tsum_{t=1}^T \rbr{\sqrt{K}\alpha_t z_t^2 + \beta_t z_t d_t} + 2\sqrt{K}\alpha_T^{-1} + \log(K)\beta_T^{-1}} + \tsum_{t=1}^T \Pr(|S_t^0| = C).\\
    &\myle{a} 
    \E\sbr{4\sqrt{K}\sqrt{\tsum_{t=1}^T z_t^2 + C\mu_{\max}^2} + 3\sqrt{\log(K)}\sqrt{\tsum_{t=1}^T z_t d_t + C\dmax\mu_{\max}}} + \delta T \\
    &\myle{b} 
    4\sqrt{K}\sqrt{\tsum_{t=1}^T \mu_t + C\mu_{\max}^2} + 3\sqrt{\log(K)}\sqrt{\tsum_{t=1}^T d_t + C\dmax\mu_{\max}} + \delta T\\
    &\le 4\sqrt{K}\sqrt{T + \tfrac{(1+\alpha)\nu_T}{C}(D+T)} + 3\sqrt{D\log(K)}\\
    &+ 7\sqrt{C \mu_{\max}(K\mu_{\max} + \log(K)\dmax)} + \delta T
\end{align*}
where (a) applies Lemma~\ref{lem:overflow-bound} and Lemma~\ref{lem:sqrt-frac-ineq} for our choice of the learning rates, (b) applies Jensen's inequality and substitutes $\E[z_t] \le 1$ and $\E[z_t^2] \le \mu_t$. Similarly, in the full-information regime, from Theorem~\ref{thm:ftrl-scheduler-no-sigmas}, we have
\begin{align*}
    \eR_T 
    &\le \E\sbr{\tsum_{t=1}^T \rbr{\beta_t z_t^2 + \beta_t z_t d_t} + \log(K)\beta_T^{-1}} + \tsum_{t=1}^T \Pr(|S_t^0| = C)\\
    &\myle{c} \E\sbr{3\sqrt{\log(K)}\sqrt{\tsum_{t=1}^T z_t(z_t + d_t) + C \mu_{\max} (\mu_{\max} + \dmax)}} + \delta T\\ 
    &\myle{d} 3\sqrt{\log(K)}\sqrt{\tsum_{t=1}^T (\mu_t + d_t) + C \mu_{\max} (\mu_{\max} + \dmax)} + \delta T\\
    &\le 3\sqrt{\log(K)} \sqrt{(D + T)(1+\tfrac{(1+\alpha)\nu_T}{C})} + 3\sqrt{C\mu_{\max}\log(K)(\mu_{\max} + \dmax)} + \delta T,
\end{align*}
where (c) applies Lemma~\ref{lem:overflow-bound} and Lemma~\ref{lem:sqrt-frac-ineq} for our choice of the learning rates, (d) applies Jensen's inequality and substitutes $\E[z_t] \le 1$ and $\E[z_t^2] \le \mu_t$.
\end{proof}

\section{Non-Clairvoyant Non-\Preemptive Delay Scheduling}\label{app:NCNP-scheduling}

For completeness, we also consider the Non-Clairvoyant and Non-\Preemptive Delay Scheduling. The restrictions of this framework put the player at a great disadvantage. For instance, any unlucky scheduling of a round with an $\Omega(T)$-long delay effectively removes one unit of capacity from the player for the rest of the game. Thus, in the absence of preemption, runtime information about delays in this framework is even more limited than in the Non-Clairvoyant Preemptive framework, for which we already require knowledge of $\dmax$. 

Nonetheless, given prior knowledge of either $T$ and $D$ or $\dmax$  (which could be the vacuous upper bound $ \dmax = T$), we can derive several upper bounds on expected regret using the Scheduling and Batching techniques from Sections~\ref{sec:scheduling-and-learning} and \ref{sec:batching}, as stated in Corollaries~\ref{cor:NCNP-known-TD} and \ref{cor:NCNP-known-dmax}, respectively.

We first prove Theorem~\ref{thm:NCNP-fixed-p}, from which Corollary~\ref{cor:NCNP-known-TD} directly follows.

\begin{theorem}\label{thm:NCNP-fixed-p}
    Suppose that Algorithm~\ref{alg:FTRL-with-scheduler} is run with Scheduler~\ref{alg:scheduler-bernoulli} with fixed probabilities $p_t = p$ and learning rates $\alpha_t = \alpha, \beta_t = \beta$. Then, in the bandit regime, we have:
    \begin{equation*}
        \eR_T \le \sqrt{K} T \alpha p^{-1} + \beta D + 2 \sqrt{K}\alpha^{-1} + \log(K)\beta^{-1}  + \tfrac{pD}{C},
    \end{equation*}
    and in the full-information regime:
    \begin{equation*}
        \eR_T \le \beta T p^{-1} + \beta D + \log(K)\beta^{-1} + \tfrac{pD}{C}.
    \end{equation*}
\end{theorem}
\begin{proof}
First of all, note that constant learning rates are clearly $\cF_t^{\cS}$-measurable. The Bernoulli scheduler (Scheduler~\ref{alg:scheduler-bernoulli}) is also clearly quantified by the sequence $p_t = p$. Therefore, we can apply Theorem~\ref{thm:ftrl-scheduler}. Also, since $\E[Z_t] \le p$, applying Markov's inequality gives, for every $t \in [T]$, $$\Pr(|S_t^0| = C) \le \Pr\rbr{\tsum_{s\in \cW_t} Z_t \ge C} \le \frac{p\sigma_t}{C}.$$
Then, applying Theorem~\ref{thm:ftrl-scheduler-no-sigmas} (modification of Theorem~\ref{thm:ftrl-scheduler}) in the bandit regime, grants us
\begin{align*}
    \eR_T 
    &\le  \E\sbr{\tsum_{t=1}^T \rbr{\sqrt{K}\alpha \frac{Z_t}{p^2} + \beta \frac{Z_t}{p} d_t} + 2\sqrt{K}\alpha^{-1} + \log(K)\beta^{-1}} + \tsum_{t=1}^T \Pr(|S_t^0| = C)\\
    &\le \sqrt{K} T\alpha/p + \beta\tsum_{t=1}^T d_t + 2\sqrt{K}\alpha^{-1} + \log(K)\beta^{-1} + \tsum_{t=1}^T \tfrac{p\sigma_t}{C}\\
    &= \sqrt{K} T \alpha p^{-1} + \beta D + 2 \sqrt{K}\alpha^{-1} + \log(K)\beta^{-1} + \tfrac{pD}{C}.
\end{align*}
and in the full-information regime:
\begin{align*}
    \eR_T 
    &\le \E\sbr{\tsum_{t=1}^T \rbr{\beta \frac{Z_t}{p^2} + \beta \frac{Z_t}{p} d_t} + \log(K)\beta^{-1}} + \tsum_{t=1}^T \Pr(|S_t^0| = C)\\
    &\le \beta T / p + \beta \tsum_{t=1}^T d_t + \log(K)\beta^{-1} + \tfrac{p\sigma_t}{C}\\
    &\le \beta T p^{-1} + \beta D + \log(K)\beta^{-1} + \tfrac{pD}{C}.
\end{align*}
\end{proof}

\begin{corollary}[Scheduling approach with known $T, D$]\label{cor:NCNP-known-TD}
    In the bandit regime, suppose $C\le \tfrac{D+T}{\sqrt{TK}}$. Setting parameters as $p = \sqrt[3]{\tfrac{C^2 T K}{(D+T)^2}}$, $\alpha = \sqrt[3]{\tfrac{C\sqrt{K}}{T(D+T)}}$, and $\beta = \sqrt{\tfrac{\log(K)}{D+T}}$, the algorithm in Theorem~\ref{thm:NCNP-fixed-p} achieves a regret bound of
    \begin{equation*}
        \eR_T \le 4\sqrt[3]{\tfrac{T(D+T)K}{C}} + 2\sqrt{(D+T)\log(K)}. 
    \end{equation*}
    In the full-information regime, suppose $C\le \tfrac{T}{\sqrt{(D+T)\log(K)}}$. Setting parameters as $p = \sqrt[3]{\tfrac{C^2 T \log(K)}{(D+T)^{2}}}$ and $\beta = \sqrt[3]{\tfrac{C\log^2(K)}{T(D+T)}}$, the algorithm in Theorem~\ref{thm:NCNP-fixed-p} achieves a regret bound of:
    \begin{equation*}
        \eR_T \le 3\sqrt[3]{\tfrac{T(D+T) \log(K)}{C}} + \sqrt{(D+T)\log(K)}. 
    \end{equation*}
\end{corollary}
\begin{proof}
    Corollary~\ref{cor:NCNP-known-TD} restricts the capacity in order to ensure that the chosen probability $p$ remains within the interval $(0,1]$. Nevertheless, an algorithm designed for a smaller capacity can be trivially simulated on a larger one. Moreover, in both the bandit and full-information regimes, as the capacity approaches its restriction, the stated regret bounds converge to those of Delayed Online Learning.
    
    To establish these bounds on expected regret, substitute the chosen values of $p$, $\alpha$, and $\beta$ into the bounds from Theorem~\ref{thm:NCNP-fixed-p} for both the bandit and full-information regimes.
\end{proof}

\begin{corollary}[Batching approach with known $\dmax$]\label{cor:NCNP-known-dmax}
    Suppose $C \ge 2$ and $\dmax > 0$. Algorithm~\ref{alg:batch-partition} with batch size $b = \tceil{\tfrac{\dmax}{C-1}}$ and learning rates from Theorem~\ref{thm:batch-bandit-fullinfo} guarantees that
    \begin{align*}
        \eR_{T} \le 28\sqrt{\tfrac{T \dmax K}{C-1}} + 3\sqrt{D\log(K)}
    \end{align*}
    in the bandit regime and
    \begin{align*}
        \eR_{T} \le 24\sqrt{\tfrac{T \dmax \log(K)}{C-1}} + 3\sqrt{D\log(K)}
    \end{align*}
    in the full-information regime.
\end{corollary}
\begin{proof}
    Follows directly from Theorem~\ref{thm:batch-bandit-fullinfo}.
\end{proof}

\section{Delay Scheduling under the Expectation-Capacity Constraint}\label{app:results-expected-capacity}

In this section, we examine a variant of Delay Scheduling in which the expected size of the tracking set in each round is constrained by the expectation-capacity $\Cexp$. Provided prior knowledge of $\log(T)$, we derive the regret bounds presented in \Cref{tab:results-expected-capacity}. 

\begin{table}[H] \centering 
\resizebox{1\columnwidth}{!}{
\begin{tabular}{c|c|c} 
\hline
\multicolumn{3}{c}{\textbf{Delay Scheduling under the expectation-capacity constraint for $C_E > 0$}}\\
\hline
Framework & Regime & Regret Bounds\\ 
\hline
\multirow{2}{*}{\makecell{Clairvoyant\\ Non-\Preemptive}} & Bandit & $O\rbr{\sqrt{TK + \tfrac{\log(T)}{\Cexp}(D+T)K} + \sqrt{D\log(K)}}$\\
& Full-info & $O\rbr{\sqrt{(1+\tfrac{\log(T)}{\Cexp})(D + T)\log(K)}}$\\
\hline
\multirow{2}{*}{\makecell{Non-Clairvoyant\\ \Preemptive}} & Bandit & $O\rbr{\sqrt{TK + \tfrac{\log(T)}{\Cexp}(D+T)K} + \sqrt{D\log(K)}} + \widetilde O\rbr{\dmax(1 + \tfrac{K}{\Cexp})}$\\
& Full-info & $O\rbr{\sqrt{(1+\tfrac{\log(T)}{\Cexp})(D + T)\log(K)}} + \widetilde O\rbr{\dmax(1+\tfrac{1}{\Cexp})}$\\
\hline
\end{tabular} 
}
\caption{Regret upper bounds for Delay Scheduling under the expectation-capacity constraint.
}   
\label{tab:results-expected-capacity}
\end{table}

\noindent\textit{Proof:} For \Cref{thm:CNP-results,thm:NCP-results} to hold, it suffices for the normalization sequence $(\nu_t)$ to be non-decreasing and satisfy $\nu_t \ge 2 H_t$ for all $t \in [T]$. Then, the results in \Cref{tab:results-expected-capacity} follow directly from these theorems if we were to run their corresponding algorithms with capacity $C = \tceil{\max\{3,K\}\log(T)}$ in the bandit regime or capacity $C = \tceil{\max\{3,\log(K)\}\log(T)}$ in the full-information regime, Chernoff parameter $\alpha = 1$, and sequence $\nu_t = 2H_t \max\{1, C / \Cexp\}$, while considering $\delta = T^{-0.5}$, assuming that the expectation capacity-constraint is satisfied for this choice of parameters. 

It remains to verify this constraint. Fix arbitrary $t \in [T]$. Following a similar argument as in the proof of Lemma~\ref{lem:overflow-bound} and applying Fact~\ref{fact:proper-seq}, we obtain  
\begin{align*}
    \E[|S_t^1|] \le \tsum_{s=1}^t \Pr(\wtd_s \ge t-s) = \sum_{s=1}^t \min\calbr{1, \tfrac{C}{(1+\alpha)\nu_s (t-s+1)}} \le \tfrac{C}{(1+\alpha)\max\{1, C/\Cexp\}} < \Cexp. 
\end{align*}
Thus, the expectation-capacity constraint holds for every round. \hfill $\blacksquare$

\vspace{1cm}

Additionally, we derive matching lower bounds, up to logarithmic factors, by analyzing the fixed delays scenario and applying reduction techniques analogous to those used in Theorem~\ref{thm:fixed-delays-lb}.

\begin{theorem}\label{thm:expected-capacity-lb} Suppose $\Cexp \ge \tfrac{(d+1)K}{T}$. Then, in the bandit regime, the minimax regret of Delay Scheduling with fixed delays under the expectation-capacity constraint is of the order
    \begin{align*}
       \Omega\rbr{\sqrt{TK (1+\tfrac{d+1}{\Cexp})} + \sqrt{Td\log(K)}}.
    \end{align*}
    And in the full-information regime, regret is of the order
    \begin{align*}
       \Omega\rbr{\sqrt{(1+\tfrac{1}{\Cexp})T(d+1)\log(K)}}.
    \end{align*}
\end{theorem}

\begin{proof}
    By closely examining the proof of Theorem 30 in \cite{audibert10} (see Theorem~\ref{thm:audibert10-lb} here), we note that the lower bound for label-efficient settings remains valid even when the expected number of queries is at most $M$. In particular, equation (30) of their proof bounds $\E_0[\sum_{t=1}^T \ind(Z_t=1)]$ by $M$.

    In the Delay Scheduling game, for feedback from round $t$ to be observed, it must satisfy $t \in S_{\tau}^1$ for all $\tau \in \{t, t+1, \ldots, t+d\}$. Consequently, we have
    \begin{align*}
        \E\sbr{\tsum_{t=1}^T Z_t (d+1)} \le \sum_{t=1}^T \E\sbr{|S_t^1|} \le \Cexp T. 
    \end{align*}
    Therefore, in expectation, the player observes losses from no more than $M = \frac{\Cexp T}{d+1}$ different rounds. Note that $K \le M$ by assumption.
    
    As in the proof of Theorem~\ref{thm:fixed-delays-lb}, we use reductions from both Delayed Online Learning with with fixed delays and Label-Efficient learning, in order to derive lower bounds on the regret for both regimes. From Theorems~\ref{thm:audibert10-lb} and \ref{thm:cesa16-lb} for the bandit regime, we have
    \begin{align*}
        \eR_T 
        = \Omega\rbr{\max\calbr{\sqrt{\tfrac{T^2K}{M}}, \sqrt{TK} + \sqrt{Td\log(K)}}}
        = \Omega\rbr{\sqrt{TK(1+\tfrac{d+1}{\Cexp})} + \sqrt{Td\log(K)}}.
    \end{align*}
    And, from Theorem~\ref{thm:audibert10-lb} and \cite{weinberger2002} for the full-information regime, we have
    \begin{align*}
        \eR_T 
        = \Omega\rbr{\max\calbr{\sqrt{\tfrac{T^2\log(K)}{M}}, \sqrt{T(d+1)\log(K)}}}
        = \Omega\rbr{\sqrt{(1+\tfrac{1}{\Cexp})T(d+1)\log(K)}}.
    \end{align*}
\end{proof}

\end{document}